\documentclass[11pt]{article}

\usepackage{algorithm}
\usepackage{algorithmicx}
\usepackage[noend]{algpseudocode}
\usepackage[colorlinks=true,allcolors=blue,urlcolor={black}]{hyperref}

\usepackage{jinshuomacros}
\usepackage{amsmath}
\usepackage{amssymb,amsmath,amsthm,amsfonts}
\usepackage{thmtools,bm}
\usepackage{mathtools}
\usepackage{mathrsfs}
\usepackage{tikz}
\usepackage{graphicx}
\graphicspath{{figs/}}
\usepackage{caption}
\usepackage{booktabs}
\usepackage[numbers]{natbib}
\usepackage{smartdiagram}
\usesmartdiagramlibrary{additions} 
\usepackage{tikz}
\usetikzlibrary{arrows,quotes}

\newcommand{\PP}{\mathbb{P}}

\newtheorem{theorem}{Theorem}
\newtheorem{othertheorem}{othertheorem}[section]
\newtheorem{lemma}[othertheorem]{Lemma}

\newtheorem{proposition}[othertheorem]{Proposition}

\newtheorem{definition}[othertheorem]{Definition}

\title{Deep Learning with Gaussian Differential Privacy}
\author{Zhiqi Bu\thanks{Graduate Group in Applied Mathematics and Computational Science. Email: {\tt zbu@sas.upenn.edu}.}
	\and Jinshuo Dong\thanks{Graduate Group in Applied Mathematics and Computational Science. Email: {\tt jinshuo@sas.upenn.edu}. }
	\and Qi Long\thanks{Department of Biostatistics, Epidemiology and Informatics. Email: {\tt qlong@pennmedicine.upenn.edu}.}
	\and Weijie J.~Su\thanks{Department of Statistics. Email: {\tt suw@wharton.upenn.edu}.}
	}
\date{}

\begin{document}
\maketitle

{\centering
\vspace*{-0.6cm}
\textit{University of Pennsylvania}\\
\par\bigskip
November 2019; Revised July 2020
\par
}

\begin{abstract}

Deep learning models are often trained on datasets that contain sensitive information such as individuals' shopping transactions, personal contacts, and medical records. An increasingly important line of work therefore has sought to train neural networks subject to privacy constraints that are specified by differential privacy or its divergence-based relaxations. These privacy definitions, however, have weaknesses in handling certain important primitives (composition and subsampling), thereby giving loose or complicated privacy analyses of training neural networks. In this paper, we consider a recently proposed privacy definition termed \textit{$f$-differential privacy} \cite{dong2019gaussian} for a refined privacy analysis of training neural networks. Leveraging the appealing properties of $f$-differential privacy in handling composition and subsampling, this paper derives analytically tractable expressions for the privacy guarantees of both stochastic gradient descent and Adam used in training deep neural networks, without the need of developing sophisticated techniques as~\cite{deep} did. Our results demonstrate that the $f$-differential privacy framework allows for a new privacy analysis that improves on the prior analysis~\cite{deep}, which in turn suggests tuning certain parameters of neural networks for a better prediction accuracy without violating the privacy budget. These theoretically
derived improvements are confirmed by our experiments in a range of tasks in image classification, text classification, and recommender systems. Python code to calculate the privacy cost for these experiments is publicly available in the \texttt{TensorFlow Privacy} library.




\end{abstract}

\section{Introduction}
\label{sec:introduction}


In many applications of machine learning, the datasets contain sensitive information about individuals such as location, personal contacts, media consumption, and medical records. Exploiting the output of the machine learning algorithm, an adversary may be able to identify some individuals in the dataset, thus presenting serious privacy concerns. This reality gave rise to a broad and pressing call for developing privacy-preserving data analysis methodologies. Accordingly, there have been numerous investigations in the scholarly literature of many fields---statistics, cryptography, machine learning, and law---for the protection of privacy in data analysis.

Along this line, research efforts have repeatedly suggested the necessity of a \textit{rigorous} and \textit{versatile} definition of privacy. Among other things, researchers have questioned whether the use of a privacy definition gives interpretable privacy guarantees, and if so, whether this privacy definition allows for high accuracy of the private model among alternative definitions. In particular, anonymization as a syntactic and \textit{ad-hoc} privacy concept has been shown to generally fail to guarantee privacy. Examples include the identification of a homophobic individual in the anonymized Netflix Challenge dataset \cite{netflix} and the identification of the health records of the then Massachusetts governor in public anonymized medical datasets
\cite{sweeney1997weaving}. 

In this context, $(\epsilon, \delta)$-\textit{differential privacy} (DP) arose as a mathematically rigorous definition of privacy~\cite{DMNS06}. Today, this definition has developed into a firm foundation of private data analysis, with its applications deployed by Google \cite{rappor}, Apple \cite{apple}, Microsoft \cite{microsoft}, and the US Census Bureau~\cite{census}. Despite its impressive popularity in both the scholarly literature and the industry, $(\epsilon, \delta)$-DP is not versatile enough to handle \textit{composition}, which is perhaps the most fundamental primitive in statistical privacy. For example, the training process of deep neural networks is in effect the composition of many primitive building blocks known as stochastic gradient descent (SGD). Under a modest privacy budget in the $(\epsilon, \delta)$-DP sense, however, it was not clear how to maintain a high prediction accuracy of deep learning. This requires a tight privacy analysis of composition in the $(\epsilon, \delta)$-DP framework. Indeed, the analysis of the privacy costs in deep learning was refined only recently using a sophisticated technique called the moments accountant \cite{deep}.


Ideally, we hope to have a privacy definition that allows for refined privacy analyses of various algorithms in a principled manner, \textit{without} resorting to sophisticated techniques. Having a refined privacy analysis not only enhances the trustworthiness of the models but can also be leveraged to improve the prediction accuracy by trading off privacy for utility. One possible candidate is \textit{$f$-differential privacy}, a relaxation of $(\epsilon, \delta)$-DP that was recently proposed by Dong, Roth, and Su \cite{dong2019gaussian}. This new privacy definition faithfully retains the hypothesis testing interpretation of differential privacy and can losslessly reason about common primitives associated with differential privacy, including composition, privacy amplification by subsampling, and group privacy. In addition, $f$-DP includes a canonical single-parameter family that is referred to as \textit{Gaussian differential privacy} (GDP). Notably, GDP is the focal privacy
definition due to a central limit theorem that states that the privacy guarantees of the composition of private algorithms are approximately equivalent to telling apart two shifted normal distributions.



The main results of this paper show that $f$-DP offers a rigorous and versatile framework for developing private deep learning methodologies\footnote{It is noteworthy that training deep learning models has served as an important benchmark in testing a privacy definition since the tightness of its privacy analysis crucially depends on whether the definition can tightly account for composition and subsampling.}. Our guarantee provides protection against an attacker with knowledge of the network architecture as well as the model parameters, which is in the same spirit as~\cite{shokri2015privacy,deep}. In short, this paper delivers the following messages concerning $f$-DP:
\begin{enumerate}
\item[]
\textbf{Closed-form privacy bounds.} In the $f$-DP framework, the overall privacy loss incurred in training neural networks admits an amenable closed-form expression. In contrast, the privacy analysis via the moments accountant must be done by numerical computation~\cite{deep}, and the implicit nature of this earlier approach can hinder our understanding of how the tuning parameters affect the privacy bound. This is discussed in Section~\ref{sec:noisy-sgd}.


\item[] \textbf{Stronger privacy guarantees.} The $f$-DP approach gives stronger privacy guarantees than the earlier approach~\cite{deep}, even in terms of $(\epsilon, \delta)$-DP. This improvement is due to the use of the central limit theorem for $f$-DP, which accurately captures the privacy loss incurred at each iteration in training the deep learning models. This is presented in \Cref{sec:conn-with-moments} and illustrated with numerical experiments in \Cref{sec:clt-boosts-perf}.


\item[] \textbf{Improved prediction accuracy.} Leveraging the stronger privacy guarantees provided by $f$-DP, we can trade a certain amount of privacy for an improvement in prediction performance. This can be realized, for example, by appropriately reducing the amount of noise added during the training process of neural networks so as to match the target privacy level in terms of $(\epsilon, \delta)$-DP. See \Cref{sec:conn-with-moments} and \Cref{sec:clt-impr-perf} for the development of this utility improvement.



\end{enumerate}


The remainder of the paper is structured as follows. In Section~\ref{sec:related-work} we provide a brief review of related literature. Section~\ref{sec:preliminaries} introduces $f$-DP and its basic properties at a minimal level. Next, in Section~\ref{sec:implementation} we analyze the privacy cost of training deep neural networks in terms of $f$-DP and compare it to the privacy analysis using the moments accountant. In Section~\ref{sec:results}, we present numerical experiments to showcase the superiority of the $f$-DP approach to private deep learning in terms of test accuracy and privacy guarantees. The paper concludes with a discussion in Section~\ref{sec:discussion}.


\subsection{Related Work}
\label{sec:related-work}

There are continued efforts to understand how privacy degrades under composition. Developments along this line include the basic composition theorem and the advanced composition theorem~\cite{approxdp,boosting}. In a pioneering work, \cite{KOV} obtained an optimal composition theorem for $(\epsilon, \delta)$-DP, which in fact served as one of the motivations for the $f$-DP work \cite{dong2019gaussian}. However, it is \#P hard to compute the privacy bounds from their composition theorem~\cite{complexity}. More recently, \cite{expcomposition} derived sharp composition bounds on the overall privacy loss for exponential mechanisms.

From a different angle, a substantial recent effort has been devoted to relaxing differential privacy using divergences of probability distributions to overcome the weakness of $(\epsilon, \delta)$-DP in handling composition \cite{concentrated,concentrated2,renyi,tcdp}. Unfortunately, these relaxations either lack a privacy amplification by subsampling argument or present a quite complex argument that is difficult to use \cite{balle2018privacy,wang2018subsampled}. As subsampling is inherently used in training neural networks, therefore, it is difficult to directly apply these relaxations to the privacy analysis of deep learning.

To circumvent these technical difficulties associated with $(\epsilon, \delta)$-DP and its divergence-based relaxations, Abadi et al.~\cite{deep} invented a technique termed the moments accountant to track detailed information of the privacy loss in the training process of deep neural networks. Using the moments accountant, their analysis significantly improves on earlier privacy analysis of SGD \cite{chaudhuri2011differentially,song2013stochastic,bassily2014private,shokri2015privacy,zhang2017efficient} and allows for meaningful privacy guarantees for deep learning trained on realistically sized datasets. This technique has been extended to a variety of situations by follow-up work \cite{mcmahan2018general,pichapati2019adaclip}. In contrast, our approach to private deep learning in the $f$-DP framework leverages some powerful tools of this new privacy definition, nevertheless providing a sharper privacy analysis, as seen both theoretically and empirically in Sections~\ref{sec:implementation} and \ref{sec:results}.

For completeness, we remark that different approaches have been put forward to incorporate privacy considerations into deep learning, without leveraging the iterative and subsampling natures of training deep learning models. This line of work includes training a private model by an ensemble of ``teacher'' models~\cite{papernot2018scalable,papernot2016semi}, the development of noised federated averaging algorithms~\cite{mcmahan2017learning}, and analyzing privacy costs through the lens of the optimization landscape of neural networks~\cite{xiang2019differentially}.




\newcommand{\Sample}{\mathtt{Sample}}

\section{Preliminaries}
\label{sec:preliminaries}

\subsection{$f$-Differential Privacy} 
\label{sub:differential_privacy}

In the differential privacy framework, we envision an adversary that is well-informed about the dataset except for a single individual, and the adversary seeks to determine whether this individual is in the dataset on the basis of the output of an algorithm. Roughly speaking, the algorithm is considered private if the adversary finds it hard to determine the presence or absence of any individual.

Informally, a dataset can be thought of as a matrix, whose rows each contain one individual's data. Two datasets are said to be \textit{neighbors} if one can be derived by discarding an individual from the other. As such, the sizes of neighboring datasets differ by one\footnote{Alternatively, the neighboring relationship can be defined for datasets of the same size and differing by one individual.}. Let $S$ and $S'$ be neighboring datasets, and $\epsilon \ge 0, 0 \le \delta \le 1$ be two numbers, and denote by $M$ a (randomized) algorithm that takes as input a dataset. 


\begin{definition}[\cite{DMNS06,approxdp}] \label{def:eps_delta_DP}
A (randomized) algorithm	$M$ gives $(\ep,\delta)$-differential privacy if for any pair of neighboring datasets $S,S'$ and any event $E$,
	\begin{align*}
		\PP(M(S)\in E)\leqslant \e^\ep \PP(M(S')\in E) +\delta.
	\end{align*}
\end{definition}


To achieve privacy, the algorithm $M$ is necessarily randomized, whereas the two datasets in \Cref{def:eps_delta_DP} are \textit{deterministic}. This privacy definition ensures that, based on the output of the algorithm, the adversary has a limited (depending on how small $\epsilon, \delta$ are) ability to identify the presence or absence of any individual, regardless of whether any individual opts in to or opts out of the dataset.


In essence, the adversary seeks to tell apart the two probability distributions $M(S)$ and $M(S')$ using a single draw. In light of this observation, it is natural to interpret what the adversary does as testing two simple hypotheses:
\[
H_0: \text{the true dataset is } S \quad\text{ versus }\quad H_1: \text{the true dataset is } S'.
\]
The connection between differential privacy and hypothesis testing was, to our knowledge, first noted in \cite{wasserman_zhou}, and was later developed in \cite{KOV,liu2019investigating,balle2019hypothesis}. Intuitively, privacy is well guaranteed if the hypothesis testing problem is \textit{hard}. Following this intuition, the definition of $(\epsilon, \delta)$-DP essentially uses the \textit{worst-case} likelihood ratio of the distributions $M(S)$ and $M(S')$ to measure the hardness of testing the two simple hypotheses.



Is there a more \textit{informative} measure of the hardness? In~\cite{dong2019gaussian}, the authors propose to use the \textit{trade-off} between type I error (the probability of erroneously rejecting $H_0$ when $H_0$ is true) and type II error (the probability of erroneously accepting $H_0$ when $H_1$ is true) in place of a few privacy parameters in $(\epsilon, \delta)$-DP  or divergence-based DP definitions. To formally define this new privacy definition, let $P$ and $Q$ denote the distributions of $M(S)$ and $M(S')$, respectively, and let $\phi$ be any (possibly randomized) rejection rule for testing $H_0: P$ against $H_1: Q$. With these in place, \cite{dong2019gaussian} defines the \textit{trade-off function} of $P$ and $Q$ as
\begin{align*}
 T(P,Q):[0,1] & \mapsto [0,1]\\
                \alpha & \mapsto \inf_{\phi}\{1 - \E_Q[\phi]: \E_P[\phi]\leqslant \alpha\}.
\end{align*}
Above, $\E_P[\phi]$ and $1-\E_Q[\phi]$ are type I and type II errors of the rejection rule $\phi$, respectively. Writing $f = T(P,Q)$, the definition says that $f(\alpha)$ is the minimum type II error among all tests at significance level $\alpha$. Note that the minimum can be achieved by taking the likelihood ratio test, according to the Neymann--Pearson lemma. As is self-evident, the larger the trade-off function is, the more difficult the hypothesis testing problem is (hence more privacy). This motivates the following privacy definition.
\begin{definition}[\cite{dong2019gaussian}] \label{def:fDP}
A (randomized) algorithm $M$ is $f$-differentially private if
\[
T\big(M(S), M(S')\big) \ge f
\]
for all neighboring datasets $S$ and $S'$.
\end{definition}
In this definition, both $T$ and $f$ are functions that take $\alpha\in[0,1]$ as input and the inequality holds pointwise for all $0 \le \alpha \le 1$, and we abuse notation by identifying $M(S)$ and $M(S')$ with their associated distributions. This privacy definition is easily interpretable due to its inherent connection with the hypothesis testing problem. By adapting a result due to Wasserman and Zhou~\cite{wasserman_zhou}, $(\epsilon, \delta)$-DP is a special instance of $f$-DP in the sense that an algorithm is $(\epsilon, \delta)$-DP if and only if it is $f_{\epsilon, \delta}$-DP with
\begin{equation}\label{eq:e_d_f}
f_{\ep, \delta}(\alpha) = \max\left\{ 0,1 - \delta - \e^\ep \alpha, \e^{-\ep}(1-\delta-\alpha) \right\}.
\end{equation}
The more intimate relationship between the two privacy definitions is that they are \textit{dual} to each other: briefly speaking, $f$-DP ensures $(\epsilon, \delta(\epsilon))$-DP with $\delta(\epsilon) = 1 + f^*(-\e^{\epsilon})$ for every $\epsilon \ge 0$\footnote{Here, $f^*$ is the convex conjugate, which is defined as $f^*(x) = \sup_{\alpha} \alpha x - f(\alpha)$.}.


Next, we define a single-parameter family of privacy definitions within the $f$-DP class for a reason that will be apparent later. Let $G_\mu:= T\big(\N(0,1), \N(\mu,1)\big)$ for $\mu \ge 0$. Note that this trade-off function admits a closed-form expression $G_\mu(\alpha)=\Phi(\Phi^{-1}(1-\alpha)-\mu)$, where $\Phi$ is the cumulative distribution function of the standard normal distribution.
\begin{definition}[\cite{dong2019gaussian}] \label{def:GDP}
A (randomized) algorithm $M$ is $\mu$-Gaussian differentially private (GDP) if
\[
T\big(M(S), M(S')\big) \ge G_\mu
\]
for all neighboring datasets $S$ and $S'$.
\end{definition}


In words, $\mu$-GDP says that determining whether any individual is in the dataset is at least as difficult as telling apart the two normal distributions $\N(0, 1)$ and $\N(\mu, 1)$ based on one draw. The Gaussian mechanism serves as a template to achieve GDP. Consider the problem of privately releasing a univariate statistic $\theta(S)$. The Gaussian mechanism adds $\N(0, \sigma^2)$ noise to the statistic $\theta$, which gives $\mu$-GDP if $\sigma = \mathrm{sens}(\theta)/\mu$. Here the \textit{sensitivity} of $\theta$ is defined as $\mathrm{sens}(\theta) = \sup_{S, S'} | \theta(S) - \theta(S') |$, where the supremum is over all neighboring datasets.




\subsection{Properties of $f$-Differential Privacy} 
\label{sub:composition}


\paragraph{Composition.} Deep learning models are trained using the \textit{composition} of many SGD updates. Broadly speaking, composition is concerned with a sequence of analyses on the \textit{same} dataset where each analysis is informed by the explorations of prior analyses. A central question that every privacy definition is faced with is to pinpoint how the overall privacy guarantee degrades under composition. Formally, letting $M_1$ be the first algorithm and $M_2$ be the second, we define their composition algorithm $M$ as $M(S) = (M_1(S), M_2(S, M_1(S)))$. Roughly speaking, the composition is to ``release all information that is learned by the algorithms.'' Notably, the second algorithm $M_2$ can take as input the output of $M_1$ in addition to the dataset $S$. In general, the composition of more than two algorithms follows recursively.

To introduce the composition theorem for $f$-DP, \cite{dong2019gaussian} defines a binary operation $\otimes$ on trade-off functions. Given trade-off functions $f = T(P, Q)$ and $g = T(P', Q')$, let $f\otimes g = T(P\times P', Q\times Q')$. This definition depends on the distributions $P,Q,P',Q'$ only through $f$ and $g$. Moreover, $\otimes$ is commutative and associative. Now the composition theorem can be readily stated as follows. Let $M_t$ be $f_t$-DP \textit{conditionally} on any output of the prior algorithms for $t = 1, \ldots, T$. Then their $T$-fold composition algorithm is $f_1\otimes\cdots\otimes f_T$-DP. This result shows that the composition of algorithms in the $f$-DP framework is reduced to performing the $\otimes$ operation on the associated trade-off functions. As an important fact, the privacy bound $f_1\otimes\cdots\otimes f_T$ in general cannot be improved. Put more precisely, one can find an $f_t$-DP
mechanism $M_t$ for $t = 1, \ldots, T$ such that their composition is precisely $f_1\otimes\cdots\otimes f_T$-DP (see the discussion following Theorem 4 in \cite{dong2019gaussian}).



More profoundly, a \textit{central limit theorem} phenomenon arises in the composition of many ``very private'' $f$-DP algorithms in the following sense: the trade-off functions of small privacy leakage accumulate to $G_{\mu}$ for some $\mu$ under composition. Informally, assuming each $f_t$ is very close to $\Id(\alpha) = 1 - \alpha$, which corresponds to perfect privacy, then we have \begin{tcolorbox}[boxrule=0.5pt]
\begin{equation}\label{eq:clt}
f_1\otimes f_2 \otimes \cdots \otimes f_T \text{ is approximately } G_\mu
\end{equation}
\end{tcolorbox}
\noindent if the number of iterations $T$ is sufficiently large\footnote{If $f_i = G_{\mu_i}$ for $i = 1, \ldots, T$, then the $T$-fold composition is exactly $G_{\mu}$ with $\mu = \sqrt{\sum_{i=1}^T \mu_i^2}$.}. This central limit theorem approximation is especially suitable for the privacy analysis of deep learning, where the training process typically takes at least tens of thousands of iterations. The privacy parameter $\mu$ depends on some functionals such as the Kullback--Leibler divergence of the trade-off functions.  The central limit theorem yields a very accurate approximation in the settings considered in \Cref{sec:results} (see numerical confirmation in \Cref{sec:numerical}). For a rigorous account of this central limit theorem for differential privacy, see Theorem 6 in \cite{dong2019gaussian}. We remark that a conceptually related article \cite{sommer2018privacy} developed a central limit theorem for privacy loss random variables.

At a high level, this convergence-to-GDP result brings GDP to the focal point of the family of $f$-DP guarantees, implying that GDP is to $f$-DP as normal random variables to general random variables. Furthermore, this result serves as an effective approximation tool for approximating the privacy guarantees of composition algorithms. In contrast, privacy loss cannot be losslessly tracked under composition in the $(\epsilon, \delta)$-DP framework.

\paragraph{Subsampling.} In training neural networks, the gradient at each iteration is computed from a mini-batch that is \textit{subsampled} from the training examples. Intuitively, an algorithm applied to a subsample gives stronger privacy guarantees than applied to the full sample. Looking closely, this privacy amplification is due to the fact that an individual enjoys perfect privacy if not selected in the subsample. A concrete and pressing question is, therefore, to precisely characterize how much privacy is amplified by subsampling in the $f$-DP framework. 


Consider the following sampling scheme: for each individual in the dataset $S$, include his or her datum in the subsample independently with probability $p$, which is sometimes referred to as the Poisson subsampling \cite{wang2018subsampled}. The resulting subsample is denoted by $\Sample_p(S)$. For the purpose of clearing up any confusion, we remark that the subsample $\Sample_p(S)$ has a random size and as an intermediate step is not released. Given any algorithm $M$, denote by $M \circ \Sample_p$ the subsampled algorithm. 

The subsampling theorem for $f$-DP states as follows. Let $M$ be $f$-DP, write $f_p$ for $pf+(1-p)\Id$, and denote by $f_p^{-1}$ the inverse\footnote{For any trade-off function $f = T(P, Q)$, its inverse $f^{-1} = T(Q, P)$.} of $f_p$. It is proved in \Cref{sec:numerical} that the subsampled algorithm $M \circ \Sample_p$ satisfies
\begin{equation}\label{eq:subsample}
	T\big(M\circ\Sample_p(S),M\circ\Sample_p(S')\big)\geqslant f_p
\end{equation}
if $S$ can be obtained by removing one individual from $S'$. Likewise, 
\[
T\big(M\circ\Sample_p(S'),M\circ\Sample_p(S)\big)\geqslant f_p^{-1}.
\]
As such, the two displays above say that the trade-off function of $M\circ\Sample_p$ on any neighboring datasets is lower bounded by $\min\{f_p, f_p^{-1}\}$, which however is in general non-convex and thus is not a trade-off function. This suggests that we can boost the privacy bound by replacing $\min\{f_p, f_p^{-1}\}$ with its double conjugate $\min\{f_p, f_p^{-1}\}^{**}$, which is the greatest convex lower bound of $\min\{f_p, f_p^{-1}\}$ and is indeed a trade-off function. Taken together, all the pieces show that the subsampled algorithm $M\circ\Sample_p$ is $\min\{f_p, f_p^{-1}\}^{**}$-DP. 

Notably, the privacy bound $\min\{f_p, f_p^{-1}\}^{**}$ is larger than $f$ and cannot be improved in general. In light of the above, the $f$-DP framework is flexible enough to nicely handle the analysis of privacy amplification by subsampling. In the case where the original algorithm $M$ is $(\epsilon, \delta)$-DP, this privacy bound strictly improves on the subsampling theorem for $(\epsilon, \delta)$-DP~\cite{li2012sampling}.







\newcommand{\tm}{\widetilde{M}}

\section{Algorithms and Their Privacy Analyses}
\label{sec:implementation}

\subsection{NoisySGD and NoisyAdam}
\label{sec:noisy-sgd}



SGD and Adam~\cite{kingma2014adam} are among the most popular optimizers in deep learning. Here we introduce a new privacy analysis of a private variant of SGD in the $f$-DP framework and then extend the study to a private version of Adam.


Letting $S = \{x_1, \ldots, x_n\}$ denote the dataset, we consider minimizing the empirical risk
\[
L(\theta) = \frac1n \sum_{i=1}^n \ell(\theta, x_i),
\]
where $\theta$ denotes the weights of the neural networks and $\ell(\theta, x_i)$ is a loss function. At iteration $t$, a mini-batch $I_t$ is selected from $\{1, 2, \ldots, n\}$ with subsampling probability $p$, thereby having an approximate size of $pn$. Taking learning rate $\eta_t$ and initial weights $\theta_0$, the vanilla SGD updates the weights according to
\[
\theta_{t+1} = \theta_t - \eta_t \cdot \frac1{|I_t|} \sum_{i \in I_t} \nabla_{\theta} \ell(\theta_t, x_i).
\]
To preserve privacy, \cite{chaudhuri2011differentially,song2013stochastic,bassily2014private,deep} introduce two modifications to the vanilla SGD. First, a clip step is applied to the gradient so that the gradient is in effect bounded. This step is necessary to have a finite sensitivity. The second modification is to add Gaussian noise to the clipped gradient, which is equivalent to applying the Gaussian mechanism to the updated iterates. Formally, the private SGD algorithm is described in \Cref{alg:dpsgd1}. Herein $I$ is the identity matrix and $\|\cdot\|_2$ denotes the $\ell_2$ norm. Formally, we present this in \Cref{alg:dpsgd1}, which we henceforth refer to as \texttt{NoisySGD}. As a contribution of the present paper, \texttt{NoisySGD} uses the Poisson subsampling, as opposed to the uniform sampling used in \cite{dong2019gaussian}. For completeness, we remark that there is another possible subsampling method: shuffling (randomly permuting and dividing data into folds at each epoch). We emphasize that different subsampling mechanisms produce different privacy guarantees.


\begin{algorithm}[!htb]
	\caption{\texttt{NoisySGD}}\label{alg:dpsgd1}
	\begin{algorithmic}[0]
		\State{\bf Input:} Dataset $S = \{x_1,\ldots,x_n\}$, loss function $\ell(\theta, x)$.
		\Statex \hspace{1.35cm}Parameters: initial weights $\theta_0$, learning rate $\eta_t$, subsampling probability $p$, number of
		\Statex \hspace{3.55cm}iterations $T$, noise scale $\sigma$, gradient norm bound $R$.
		\For{$t = 0, \ldots, T-1$}
		\Statex {\hspace{0.55cm}Take a Poisson subsample $I_t\subseteq \{1, \ldots, n\}$ with subsampling probability $p$}
		\For{$i\in I_t$}
		\Statex \hspace{1.1cm} {$v_t^{(i)} \gets \nabla_{\theta} \ell(\theta_t, x_i)$}		
		\Statex \hspace{1.1cm} {$\bar{v}_t^{(i)} \gets v_t^{(i)} / \max\big\{1, \|v_t^{(i)}\|_2/R\big\}$}
                \Comment{Clip gradient}
		\EndFor
		\Statex \hspace{0.55cm} {$\theta_{t+1} \gets \theta_{t} - \eta_t \cdot \frac{1}{|I_t|} \Big(\sum_{i\in I_t}\bar{v}_t^{(i)}+\sigma R \cdot \mathcal{N}(0, I)\Big)$}
       \Comment{Apply Gaussian mechanism}
		
		\EndFor
		\State {\bf Output} $\theta_T$
	\end{algorithmic}
\end{algorithm}

The analysis of the overall privacy guarantee of \texttt{NoisySGD} makes heavy use of the compositional and subsampling properties of $f$-DP. We first focus on the privacy analysis of the step that computes $\theta_{t+1}$ from $\theta_t$. Let $M$ denote the gradient update and write $\Sample_p(S)$ for the mini-batch $I_t$ (we drop the subscript $t$ for simplicity). This allows us to use $M \circ \Sample_p(S)$ to represent what \texttt{NoisySGD} does at each iteration. Next, note that adding or removing one individual would change the value of $\sum_{i\in I_t}\bar{v}_t^{(i)}$ by at most $R$ in the $\ell_2$ norm due to the clipping operation, that is, $\sum_{i\in I_t}\bar{v}_t^{(i)}$ has sensitivity $R$. Consequently, the Gaussian mechanism with noise standard deviation $\sigma R$ ensures that $M$ is $\frac1\sigma$-GDP. With a few additional arguments, in \Cref{sec:omitt-deta-crefs} we show that \texttt{NoisySGD} is $\min\{f,f^{-1}\}^{**}$-DP with $f = \big(pG_{1/\sigma}+(1-p)\Id\big)^{\otimes T}$. As a remark, it has been empirically observed in \cite{deep,bagdasaryan2019differential} that the performance of the private neural networks is not sensitive to the
choice of the clipping norm bound $R$ (see more discussion in \cite{pichapati2019adaclip,xu2019adaptive}).

To facilitate the use of this privacy bound, we now derive an analytically tractable approximation of $\min\{f,f^{-1}\}^{**}$ using the privacy central limit theorem in a certain asymptotic regime, which further demonstrates the mathematical coherence and versatility of the $f$-DP framework. The central limit theorem shows that, in the asymptotic regime where $p \sqrt{T}\to \nu$ for a constant $\nu > 0$ as $T \to \infty$,
\[
f = \big(pG_{1/\sigma}+(1-p)\Id\big)^{\otimes T}\to G_{\mu},
\]
where $\mu = \nu\sqrt{\e^{1/\sigma^2}-1}$. Thus, the overall privacy loss in the form of the double conjugate satisfies 
\begin{equation}\label{eq:g_double}
\min\{f,f^{-1}\}^{**} \approx \min\{G_{\mu}, G_{\mu}^{-1}\}^{**} = G_{\mu}^{**} = G_{\mu}.
\end{equation}
As such, the central limit theorem demonstrates that \texttt{NoisySGD} is approximately $p\sqrt{T(\e^{1/\sigma^2}-1)}$-GDP. Denoting by $B = pn$ the mini-batch size, the privacy parameter $p\sqrt{T(\e^{1/\sigma^2}-1)}$ equals $\frac{B}{n}\sqrt{T(\e^{1/\sigma^2}-1)}$. Intuitively, this reveals that \texttt{NoisySGD} gives good privacy guarantees if $B\sqrt{T}/n$ is small and $\sigma$ is not too small.


As an aside, we remark that this new privacy analysis is different from the one performed in Section 5 of \cite{dong2019gaussian}. Therein, the authors consider Algorithm~\ref{alg:dpsgd1} with uniform subsampling and obtain a privacy bound that is different from the one in the present paper.

Next, we present a private version of Adam \cite{kingma2014adam} in \Cref{alg:dpsgd2}, which we refer to as \texttt{NoisyAdam} and can be found in \cite{privacygithub}. This algorithm has the same privacy bound as \texttt{NoisySGD} in the $f$-DP framework. In short, this is because the momentum $m_t$ and $u_t$ are deterministic functions of the noisy gradients and no additional privacy cost is incurred due to the post-processing property of differential privacy. In passing, we remark that the same argument applies to AdaGrad \cite{duchi2011adaptive} and therefore it is also asymptotically GDP in the same asymptotic regime.


\begin{algorithm}[htb]
	\caption{\texttt{NoisyAdam}}\label{alg:dpsgd2}
	\begin{algorithmic}[0]
		\State{\bf Input:} Dataset $S = \{x_1,\ldots,x_n\}$, loss function $\ell(\theta, x)$.
		\Statex \hspace{1.35cm}Parameters: initial weights $\theta_0$, learning rate $\eta_t$, subsampling probability $p$, number of
		\Statex \hspace{3.55cm}iterations $T$, noise scale $\sigma$, gradient norm bound $R$, 
		\Statex \hspace{3.55cm}momentum parameters $(\beta_1,\beta_2)$, initial momentum $m_0$, 
		\Statex \hspace{3.55cm}initial past squared gradient $u_0$, and a small constant $\xi > 0$.
		\For{$t = 0, \ldots, T-1$}
		\Statex {\hspace{0.55cm}Take a Poisson subsample $I_t\subseteq \{1, \ldots, n\}$ with subsample probability $p$}
		\For{$i\in I_t$}
		\Statex {\hspace{1.1cm} $v_t^{(i)} \gets \nabla_{\theta} \ell(\theta_t, x_i)$}
		\Statex {\hspace{1.1cm} $\bar{v}_t^{(i)} \gets v_t^{(i)}/\max\big\{1,\|v_t^{(i)}\|_2/R\big\}$}
                \Comment{Clip gradient}
		\EndFor
		\Statex {\hspace{0.55cm} $\tilde{v}_t\gets\frac{1}{|I_t|}\left(\sum_{i\in I_t}\bar{v}_t^{(i)}+\sigma R\cdot \mathcal{N}(0,I)\right)$}
                \Comment{Apply Gaussian mechanism}
		\Statex {\hspace{0.55cm} $m_t\gets\beta_1 m_{t-1}+(1-\beta_1) \tilde{v}_t$}
		\Statex {\hspace{0.55cm} $u_t\gets\beta_2 u_{t-1}+(1-\beta_2) \big(\tilde{v}_t\odot\tilde{v}_t\big)$}
 \Comment{$\odot$ is the Hadamard product}
		\Statex {\hspace{0.55cm} $w_t \gets m_t/\left(\sqrt{u_t}+\xi\right)$}
                \Comment{Component-wise division}
		\Statex {\hspace{0.55cm}$\theta_{t+1} \gets \theta_{t} - \eta_t w_t$}
		\EndFor
		\State {\bf Output} $\theta_T$
	\end{algorithmic}
\end{algorithm}



\subsection{Comparisons with the Moments Accountant}
\label{sec:conn-with-moments}

It is instructive to compare the moments accountant with our privacy analysis performed in Section~\ref{sec:noisy-sgd} using the $f$-DP framework. Developed in~\cite{deep}, the moments accountant gives a tight one-to-one mapping between $\epsilon$ and $\delta$ for specifying the overall privacy loss in terms of $(\epsilon, \delta)$-DP under composition, which is beyond the reach of the advanced composition theorem~\cite{boosting}. In slightly more detail, the moments accountant uses the moment generating function of the privacy loss random variable to track the privacy loss under composition. As abuse of notation, this paper uses functions $\delta_{\mathtt{MA}} = \delta_{\mathtt{MA}}(\epsilon)$ and $\epsilon_{\mathtt{MA}} =\epsilon_{\mathtt{MA}}(\delta)$ to denote the mapping induced by the moments accountant in both directions\footnote{We omit the dependence of the functions on the specification of the composition algorithm such as $p, \sigma, T$ as in \texttt{NoisySGD} and \texttt{NoisyAdam}.}. For self-containedness, the appendix includes a formal description of the two functions.



Although \texttt{NoisySGD} and \texttt{NoisyAdam} are our primary focus, our following discussion applies to general iterative algorithms where composition must be addressed in the privacy analysis. Let algorithm $M_t$ be $f_t$-DP for $t = 1, \ldots, T$ and write $M$ for their composition. On the one hand, the moments accountant technique ensures that $M$ is $(\epsilon, \delta_{\mathtt{MA}}(\epsilon))$-DP for any $\epsilon$ or, put equivalently, is $(\epsilon_{\mathtt{MA}}, \delta)$-DP\footnote{The moments accountant can be applied in the $f$-DP framework. In fact, the moments accountant is defined via a certain moment generating function, which is equivalent to the R\'enyi divergence. The R\'enyi divergence can be uniquely deduced from a trade-off function. See Section 2.3 in
  \cite{dong2019gaussian}.}. On the other hand, the composition algorithm is $f_1 \otimes \cdots \otimes f_T$-DP from the $f$-DP viewpoint and, following from the central limit theorem \eqref{eq:clt}, this composition can be shown to
be approximately GDP in a certain asymptotic regime. For example, both \texttt{NoisySGD} and \texttt{NoisyAdam} presented in \Cref{alg:dpsgd1} and \Cref{alg:dpsgd2}, respectively,
asymptotically satisfy $\mu_{\mathtt{CLT}}$-GDP with privacy parameter 
\begin{equation}\label{eq:clt_analytical}
\mu_{\mathtt{CLT}} = p\sqrt{T(\e^{1/\sigma^2} - 1)}.
\end{equation}


In light of the above, it is tempting to ask which of the two approaches yields a sharper privacy analysis. In terms of $f$-DP guarantees, it must be the latter, which we refer to as the \texttt{CLT} approach, because the composition theorem of $f$-DP is tight and, more importantly, the privacy central limit theorem is asymptotic exact. To formally state the result, note that the moments accountant asserts that the private optimizer is $(\epsilon, \delta_{\mathtt{MA}}(\epsilon))$-DP for all $\epsilon \ge 0$, which is equivalent to $\sup_{\epsilon \ge 0} f_{\epsilon, \delta_{\mathtt{MA}}(\epsilon)}$-DP by recognizing \eqref{eq:e_d_f} (see also Proposition 2.11 in \cite{dong2019gaussian}). Roughly speaking, the
following theorem says that $\sup_{\epsilon \ge 0} f_{\epsilon, \delta_{\mathtt{MA}}(\epsilon)}$-DP (asymptotically) promises no more privacy guarantees than the bound of $\mu_{\mathtt{CLT}}$-GDP given by the \texttt{CLT} approach. This simple result is summarized by the following theorem and see \Cref{sec:omitt-deta-crefs} for a formal proof of this result.
\begin{theorem}[Comparison in $f$-DP]\label{thm:f_compare}
Assume that $p\sqrt{T}$ converges to a positive constant as $T \to \infty$. Then, both $\mathtt{NoisySGD}$ and $\mathtt{NoisyAdam}$ satisfy
\[
\limsup_{T \to \infty} \left( \sup_{\epsilon \ge 0} f_{\epsilon, \delta_{\mathtt{MA}}(\epsilon)}(\alpha) - G_{\mu_{\mathtt{CLT}}}(\alpha) \right) \le 0
\]
for every $0 \le \alpha \le 1$.
\end{theorem}

\begin{remark}
For ease of reading, we point out that, in the $(\epsilon,\delta)$-DP framework, the smaller $\epsilon, \delta$ are, the more privacy is guaranteed. In contrast, in the $f$-DP framework, the smaller $f$ is, the less privacy is guaranteed.
\end{remark}

From the $(\epsilon, \delta)$-DP viewpoint, however, the question is presently unclear. Explicitly, the duality between $f$-DP and $(\epsilon, \delta)$-DP shows that $\mu$-GDP implies $(\epsilon, \delta(\epsilon; \mu))$-DP for all $\epsilon \ge 0$, where\footnote{See Section 2.4 of \cite{dong2019gaussian} for this result. See also \cite{dwork2014analyze,balle2018improving}.}
\begin{equation}\label{eq:dual}
\delta(\epsilon; \mu) = 1 + G^*_{\mu}(-\e^{\epsilon}) = \Phi( -\frac{\varepsilon}{\mu} +\frac{\mu}{2})- \e^{\varepsilon}\Phi(- \frac{\varepsilon}{\mu} - \frac{\mu}{2}).
\end{equation}
The question is, therefore, reduced to the comparison between $\delta_{\mathtt{MA}}(\epsilon)$ and
$\delta_{\mathtt{CLT}}(\epsilon) := \delta(\epsilon; \mu_{\mathtt{CLT}})$ or, equivalently, between $\epsilon_{\mathtt{MA}}(\delta)$ and $\epsilon_{\mathtt{CLT}}(\delta) := \epsilon(\delta; \mu_{\mathtt{CLT}})$\footnote{Here, $\epsilon(\delta; \mu)$ is the inverse function of $\delta(\epsilon; \mu)$.}.


\begin{theorem}[Comparison in $(\epsilon, \delta)$-DP]\label{thm:better}
Under the assumptions of \Cref{thm:f_compare}, the $f$-DP framework gives an asymptotically sharper privacy analysis of both $\mathtt{NoisySGD}$ and $\mathtt{NoisyAdam}$ than the moments accountant in terms of $(\epsilon, \delta)$-DP. That is,
\[
\limsup_{T \to \infty} \, \left( \delta_{\mathtt{CLT}}(\epsilon) - \delta_{\mathtt{MA}}(\epsilon) \right) < 0
\]
for all $\epsilon \ge 0$. 
\end{theorem}

In words, the \texttt{CLT} approach in the $f$-DP framework allows for an asymptotically smaller $\delta$ than the moments accountant at the same $\epsilon$. It is worthwhile mentioning that the inequality in this theorem holds for any finite $T$ if $\delta$ is derived by directly applying the duality to the (exact) privacy bound $f_1 \otimes \cdots \otimes f_T$. Equivalently, the theorem says that $\limsup_{T \to \infty} \, \left( \epsilon_{\mathtt{CLT}}(\delta) - \epsilon_{\mathtt{MA}}(\delta) \right) < 0$ for any $\delta$\footnote{Write $\delta^\star_{\mathtt{CLT}} = \delta_{\mathtt{CLT}}(0)$ and set $\epsilon_{\mathtt{CLT}}(\delta) = 0$ for $\delta \ge \delta^\star_{\mathtt{CLT}}$. Apply the same adjustment for $\epsilon_{\mathtt{MA}}$.}. As such, by setting the same $\delta$ in both approaches, say $\delta =
10^{-5}$, the $f$-DP based \texttt{CLT} approach shall give a smaller value of $\epsilon$.


From a practical viewpoint, this refined privacy analysis allows us to trade privacy guarantees for improvement in utility. More precisely, recognizing the aforementioned conclusion that $\delta(\epsilon; \mu_{\mathtt{CLT}}) \equiv \delta_{\mathtt{CLT}}(\epsilon) < \delta_{\mathtt{MA}}(\epsilon)$ (for sufficiently large $T$) and that $\delta(\epsilon; \mu)$ increases as $\mu$ increases, one can find $\tilde\mu_{\mathtt{CLT}} > \mu_{\mathtt{CLT}}$ such that
\begin{equation}\label{eq:tilde_sigma}
\delta(\epsilon; \tilde\mu_{\mathtt{CLT}}) = \delta_{\mathtt{MA}}(\epsilon).
\end{equation}
Put differently, we can carefully adjust some parameters in \Cref{alg:dpsgd1} and \Cref{alg:dpsgd2} in order to let the algorithms be $\tilde\mu_{\mathtt{CLT}}$-GDP. For example, we can \textit{reduce} the scale of the added noise from $\sigma$ to a certain $\tilde\sigma
< \sigma$, which can be solved from \eqref{eq:tilde_sigma} and
\begin{equation}\label{eq:clt_tilde}
\tilde\mu_{\mathtt{CLT}} = p\sqrt{T(\e^{1/\tilde\sigma^2} - 1)}.
\end{equation}
Note that this is adapted from \eqref{eq:clt_analytical}. 




\definecolor{mycolor}{RGB}{176,224,224}
\begin{figure}[!htb]
\centering
\tikzset{
     block/.style={rectangle, draw, text width=11em,text centered, rounded corners, minimum height=3em},
     arrow/.style={-{Stealth[]}}
     }
\usetikzlibrary{positioning,arrows.meta,shapes.symbols}
\begin{tikzpicture}
    \node [fill=mycolor,block] (SGD) {\texttt{NoisyOptimizer}$(\sigma,\cdots)$};
    \node [block,right=of SGD] (eps)  {$\big(\ep,\delta_{\texttt{MA}}(\ep)\big)$-DP};
    \node [block,above=of SGD] (mu)  {$\mu_{\texttt{CLT}}$-GDP};
    \node [block,right=of mu] (eps2)  {$\big(\ep,\delta_{\texttt{CLT}}(\ep)\big)$-DP};
    \node [cloud, cloud puffs=15.7, cloud ignores aspect, minimum width=5cm, minimum height=2cm,,draw,right=of eps2] (cloud1)  {$\delta_{\texttt{MA}}(\epsilon)>\delta_{\texttt{CLT}}(\epsilon)$};
    \node [block,right=of eps] (mu2)  {$\tilde\mu_{\texttt{CLT}}$-GDP};
    \node [fill=mycolor,block,below=of eps] (SGD2)  {\texttt{NoisyOptimizer}$(\tilde\sigma,\cdots)$};
    \node [cloud, cloud puffs=13, minimum width=5cm, minimum height=2cm,,draw,left=of SGD2] (cloud2)  {$\sigma>\tilde\sigma$};
    \draw[arrow] (SGD)  --  node [right] {\texttt{CLT}}(mu);  
    \draw[arrow] (mu) --  node [above] {\texttt{Dual}}(eps2);  
	\draw [arrow] (SGD) -- node [above] {\texttt{MA}}(eps);
	\draw[arrow] (mu2)  -- node [above] {\texttt{Dual}} (eps);  
	\draw[arrow] (SGD2) -|  node [above,xshift=-20pt] {\texttt{CLT}}(mu2);
	\draw[dashed,arrow,transform canvas={yshift=-3pt}](eps)  -- node [below]{\eqref{eq:tilde_sigma}}(mu2);
	\draw[dashed,arrow,transform canvas={xshift=3pt,yshift=-3pt}](mu2)  |- node [below,xshift=-20pt]{\eqref{eq:clt_tilde}}(SGD2);
\end{tikzpicture}
\caption{An illustration of the \texttt{CLT} approach in the $f$-DP framework and the moments accountant in the $(\epsilon, \delta)$-DP framework. $\mathtt{NoisyOptimizer}(\sigma, \dots)$ using the moments accountant gives the same privacy guarantees in terms of $(\epsilon, \delta)$-DP as $\mathtt{NoisyOptimizer}(\tilde\sigma, \dots)$ using the \texttt{CLT} approach (the ellipses denote omitted parameters). Note that the duality formula \eqref{eq:dual} is used in solving $\tilde\mu_{\mathtt{CLT}}$ from \eqref{eq:tilde_sigma}.}
\label{fig:workflow}
\end{figure}
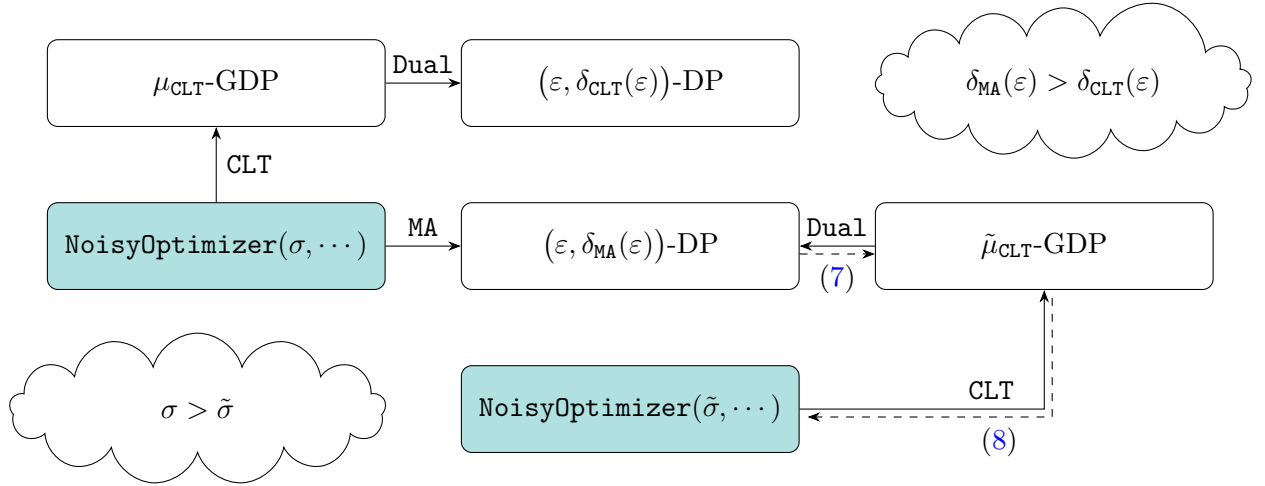





Figure~\ref{fig:workflow} shows the flowchart of the privacy analyses using the two approaches and their relationship. In addition, numerical comparisons are presented in Figure~\ref{fig:necessary2}, consistently demonstrating the superiority of the \texttt{CLT} approach. 


\begin{figure}[!htb]
	\centering
	\includegraphics[width=8cm]{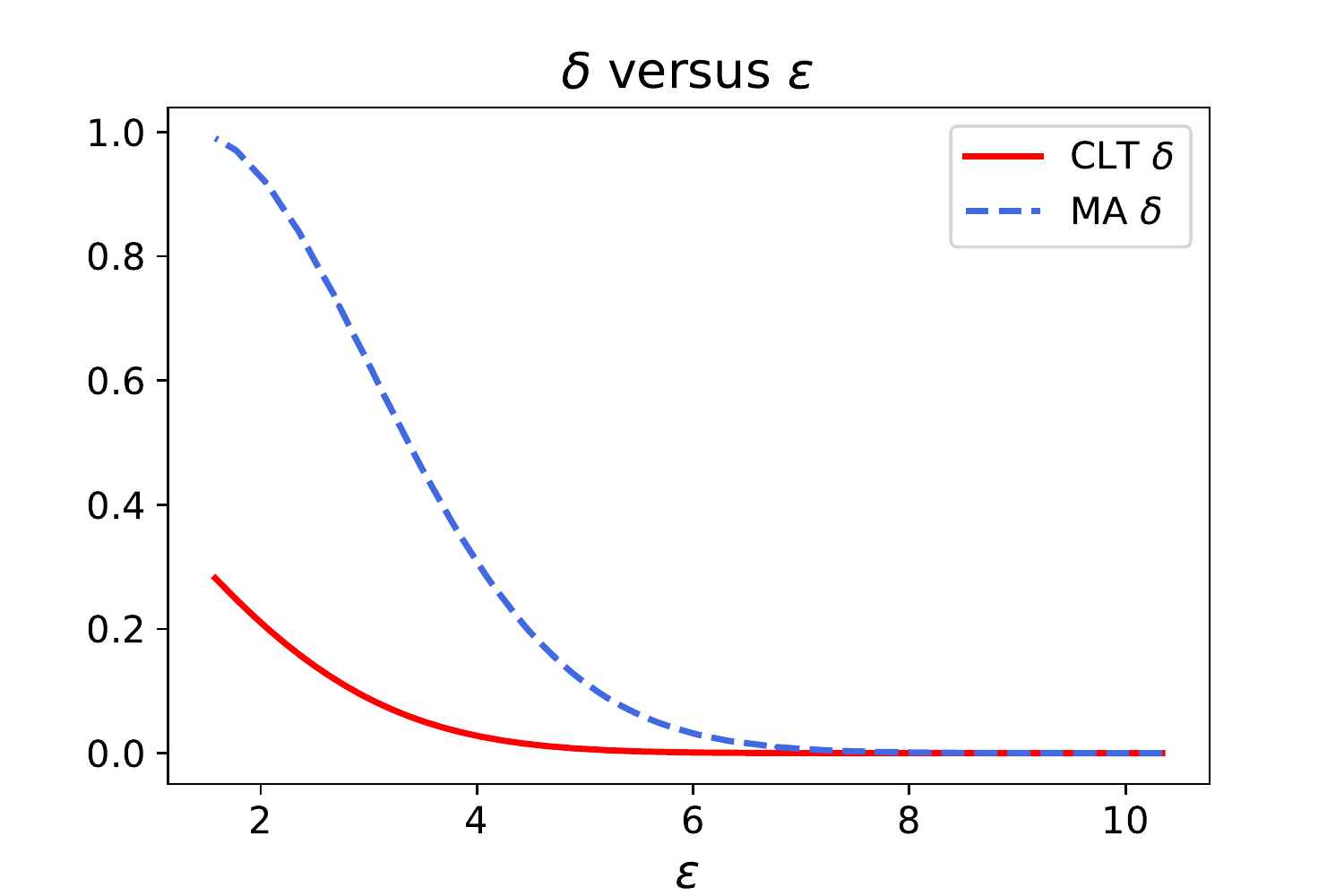}
	\hspace{-0.7cm}
	\includegraphics[width=8cm]{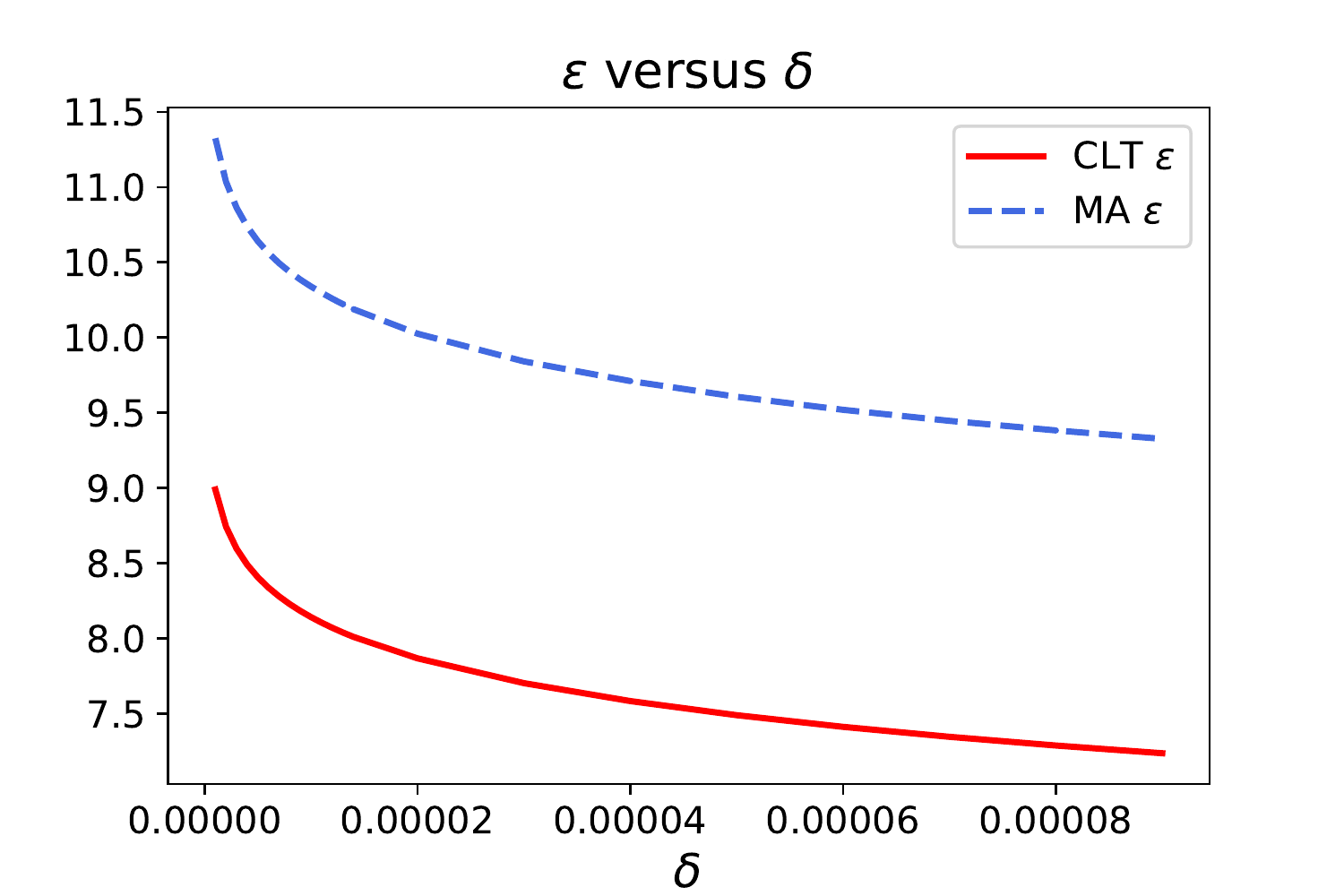}
	\hspace{-0.7cm}
	\includegraphics[width=8cm]{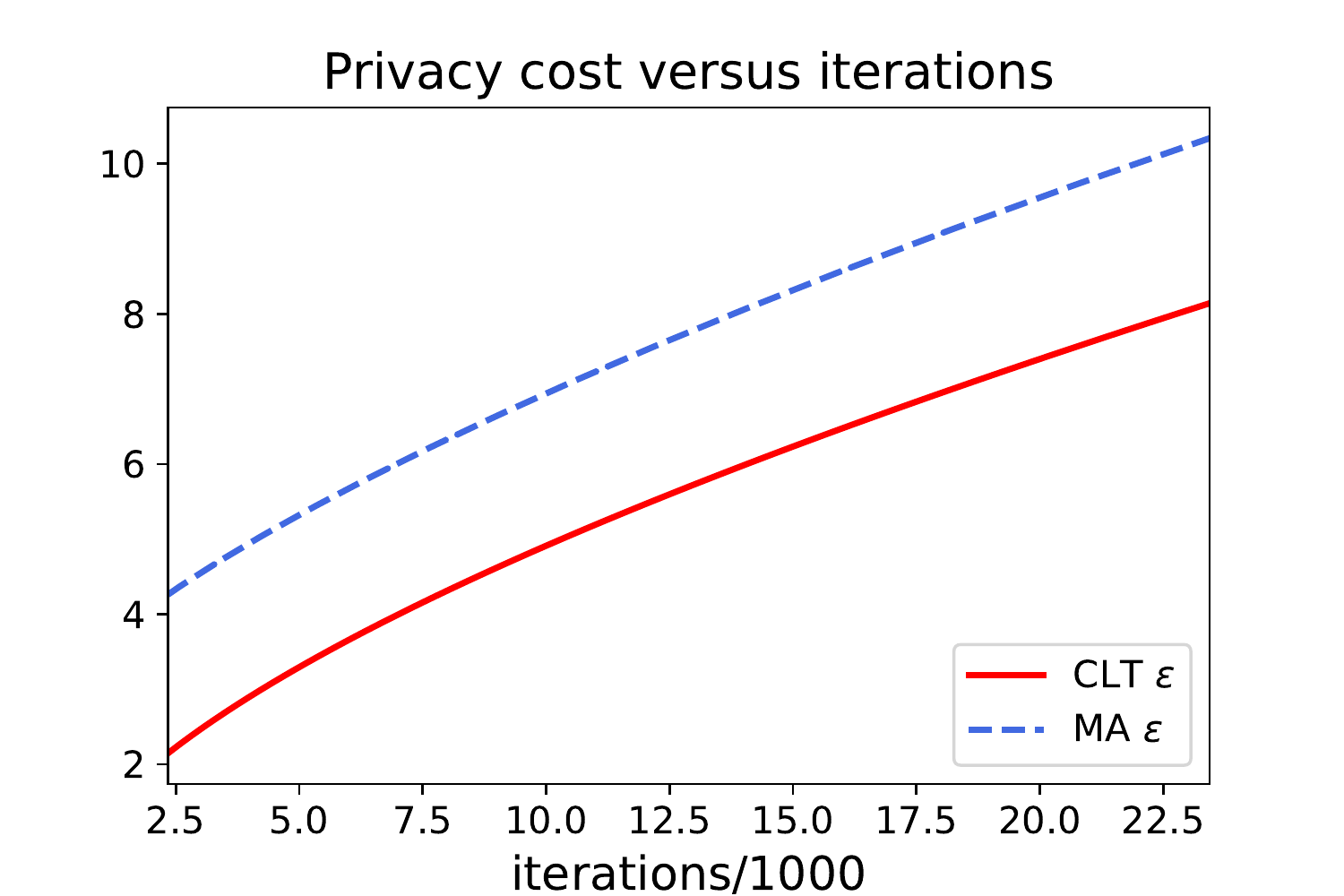}
	\hspace{-0.7cm}
	\includegraphics[width=8cm]{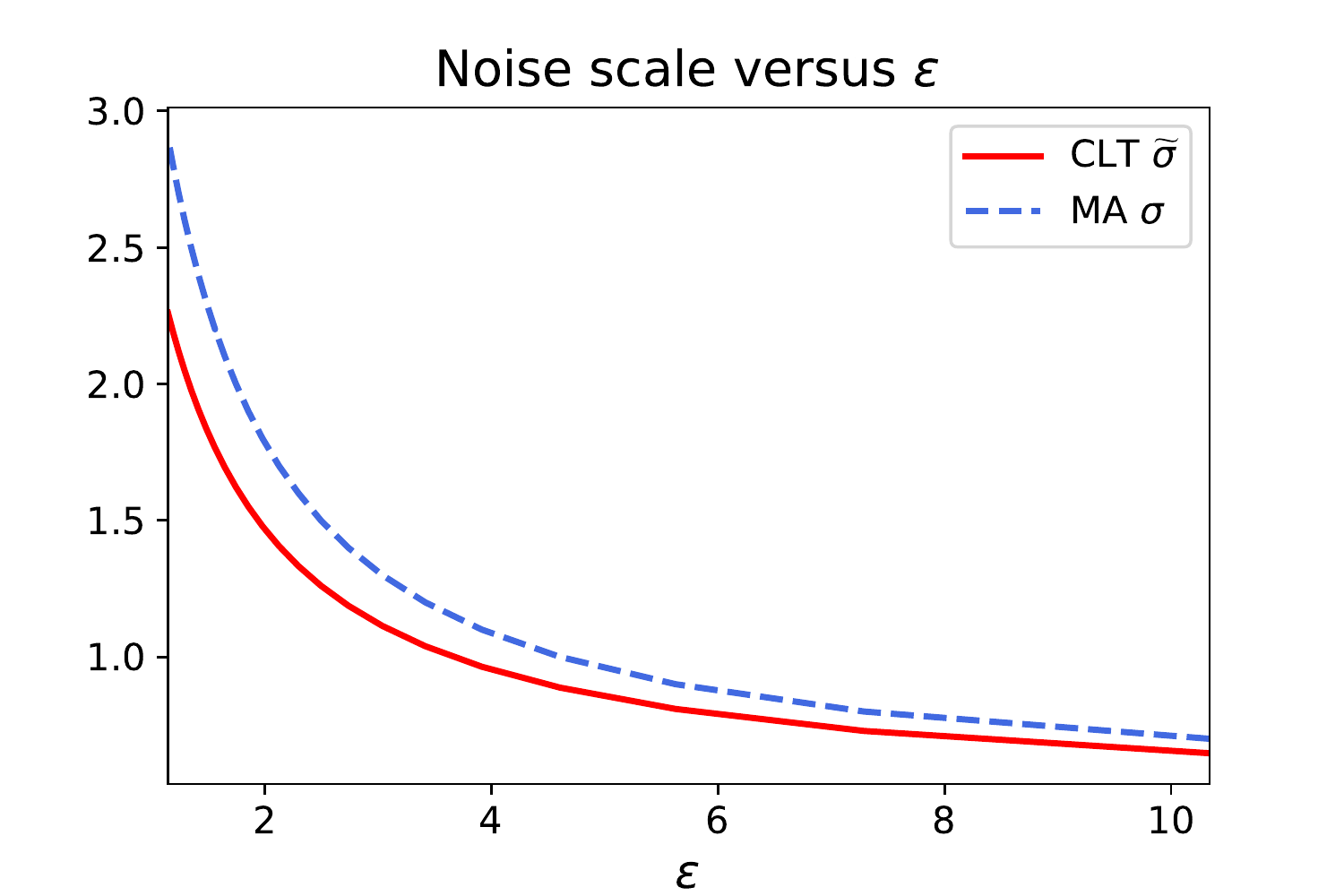}
	\caption{Tradeoffs between $\epsilon$ and $\delta$ for both \texttt{CLT} and \texttt{MA}, which henceforth denotes the moments accountant. The settings follows the MNIST experiment in \Cref{sec:results} with $\sigma=0.7, p=256/60000$. The bottom two plots assume $\delta=10^{-5}$. Note $\epsilon$ and $\delta$ in the \texttt{CLT} are related via \eqref{eq:dual} with $\mu = \mu_{\mathtt{CLT}}$. The bottom right plot is consistent with the conclusion $\sigma > \tilde\sigma$ shown in the cloud icon of \Cref{fig:workflow}.}
	\label{fig:necessary2}
\end{figure}



\section{Results}
\label{sec:results}

In this section, we use \texttt{NoisySGD} and \texttt{NoisyAdam} to train private deep learning models on datasets for tasks ranging from image classification (MNIST), text classification (IMDb movie review), recommender systems (MovieLens movie rating), to regular binary classification (Adult income). Note that these datasets all contain sensitive information about individuals, and this fact necessitates privacy consideration in the training process. Code to reproduce the results using the \texttt{TensorFlow Privacy} library is available at \url{https://github.com/tensorflow/privacy/tree/master/research/GDP_2019}.


\subsection{The $f$-DP Perspective}
\label{sec:clt-boosts-perf}

This section demonstrates the utility and practicality of the private deep learning methodologies with associated privacy guarantees in terms of $f$-DP. In \Cref{sec:clt-impr-perf}, we extend the empirical study to the $(\epsilon, \delta)$-DP framework. Throughout the experiments, the parameter $\delta$ we use always satisfies $\delta < 1/n$, where $n$ is the number of training examples. 


\paragraph{MNIST.} 
The MNIST dataset \cite{lecun-mnisthandwrittendigit-2010} contains 60,000 training images and 10,000 test images. Each image is in $28\times 28$ gray-scale representing a handwritten digit ranging from 0 to 9. We train neural networks with the same architecture (two convolutional layers followed by one dense layer) as in~\cite{privacygithub,deep} on this dataset. Throughout the experiment, we set the subsampling probability to $p=256/60000$ and use a constant learning rate $\eta$. 

\begin{table}[!htb]
\centering
\begin{tabular}{c|c|c|c|c|c|c|c}
\toprule 
$\eta$& $R$& $\sigma$& Epochs & Test accuracy (\%)&	\texttt{CLT} $\mu$ &	\texttt{CLT} $\epsilon$ & \texttt{MA} $\epsilon$ \\
\midrule 
\midrule 
0.25&1.5&1.3& 15 &95.0& 0.23&0.83& 1.19 \\
0.15&1.0&1.1& 60 & 96.6& 0.57&2.32& 3.01 \\
0.25&1.5&0.7& 45 & 97.0& 1.13&5.07& 7.10 \\
0.25&1.5&0.6& 62 & 97.6& 2.00&9.98& 13.27 \\
0.25&1.5&0.55& 68& 97.8&2.76&14.98& 18.72 \\
0.25&1.5&0.5& 100& 98.0& 4.78&31.12& 32.40\\
\bottomrule 
\end{tabular}
\caption{Experimental results for \texttt{NoisySGD} and their privacy analyses on MNIST. The accuracy is averaged over 10 independent runs. The hyperparameters in the first three rows are the same as in \cite{privacygithub}. The $\mu$ in the 6th row is calculated using \eqref{eq:clt_analytical}, which carries over to the 7th row via \eqref{eq:dual} with $\delta=10^{-5}$. The number of epochs is equal to $T \times \text{mini-batch size}/n = pT$.}
\label{tab:tradeoff}
\end{table}

Table~\ref{tab:tradeoff} displays the test accuracy of the neural networks trained by \texttt{NoisySGD} as well as the associated privacy analyses. The privacy parameters $\epsilon$ in the last two columns are both with respect to $\delta=10^{-5}$. Over all six sets of experiments with different tuning parameters, the \texttt{CLT} approach gives a significantly smaller value of $\epsilon$ than the moments accountant, which is consistent with our discussion in \Cref{sec:conn-with-moments}. The point we wish to emphasize, however, is that $f$-DP offers a much more comprehensive interpretation of the privacy guarantees than $(\epsilon, \delta)$-DP. For instance, the model from the third row preserves a decent amount of privacy since it is \textit{not} always easy to tell apart $\N(0,1)$ and
$\N(1.13, 1)$. In stark contrast, the $(\epsilon,\delta)$-DP viewpoint is too conservative, suggesting that for the \textit{same} model not much privacy is left, due to a very large ``likelihood ratio'' $\e^{\epsilon}$ in Definition~\ref{def:eps_delta_DP}: it equals
$\e^{7.10} = 1212.0$ or $\e^{5.07} = 159.1$ depending on which approach is chosen. This shortcoming of $(\epsilon, \delta)$-DP cannot be overcome by taking a larger $\delta$, which, although gives rise to a smaller $\epsilon$, would undermine the privacy guarantee from a different perspective.


\begin{figure}[!htb]
\centering
\includegraphics[width=5cm]{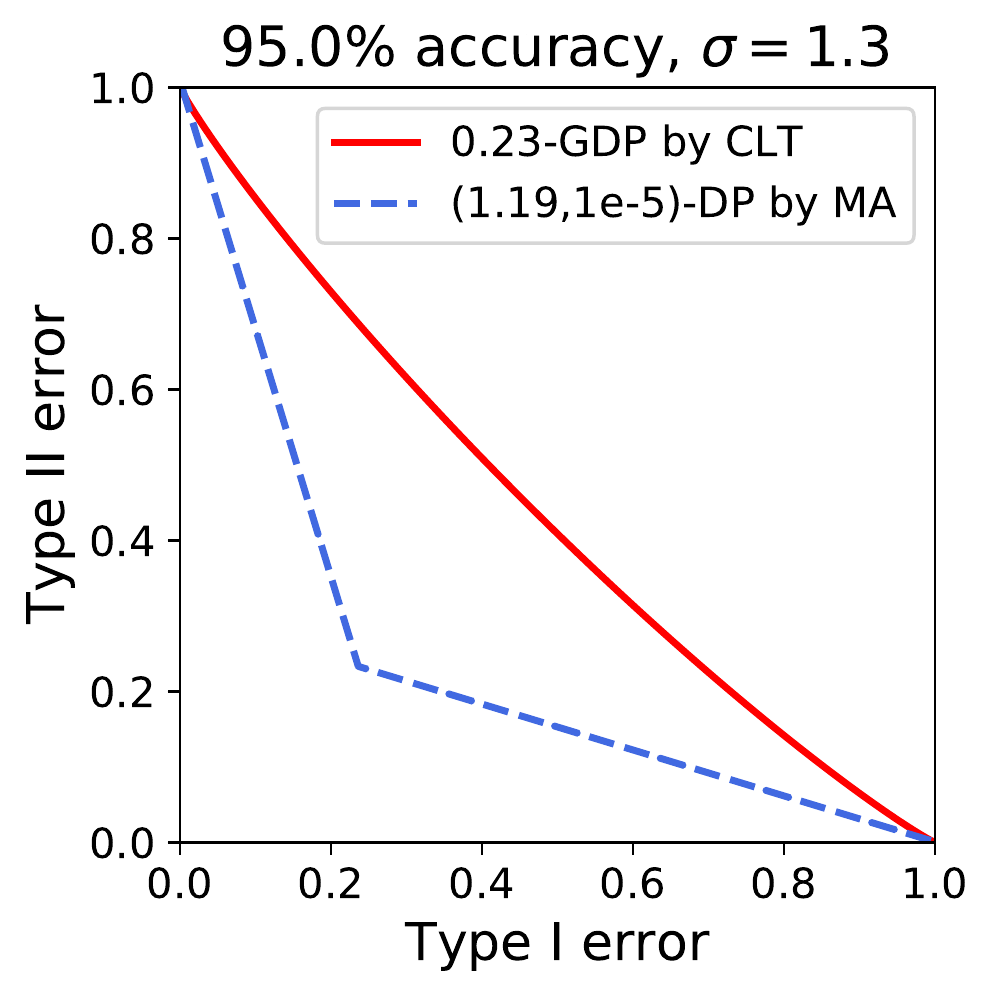}
\includegraphics[width=5cm]{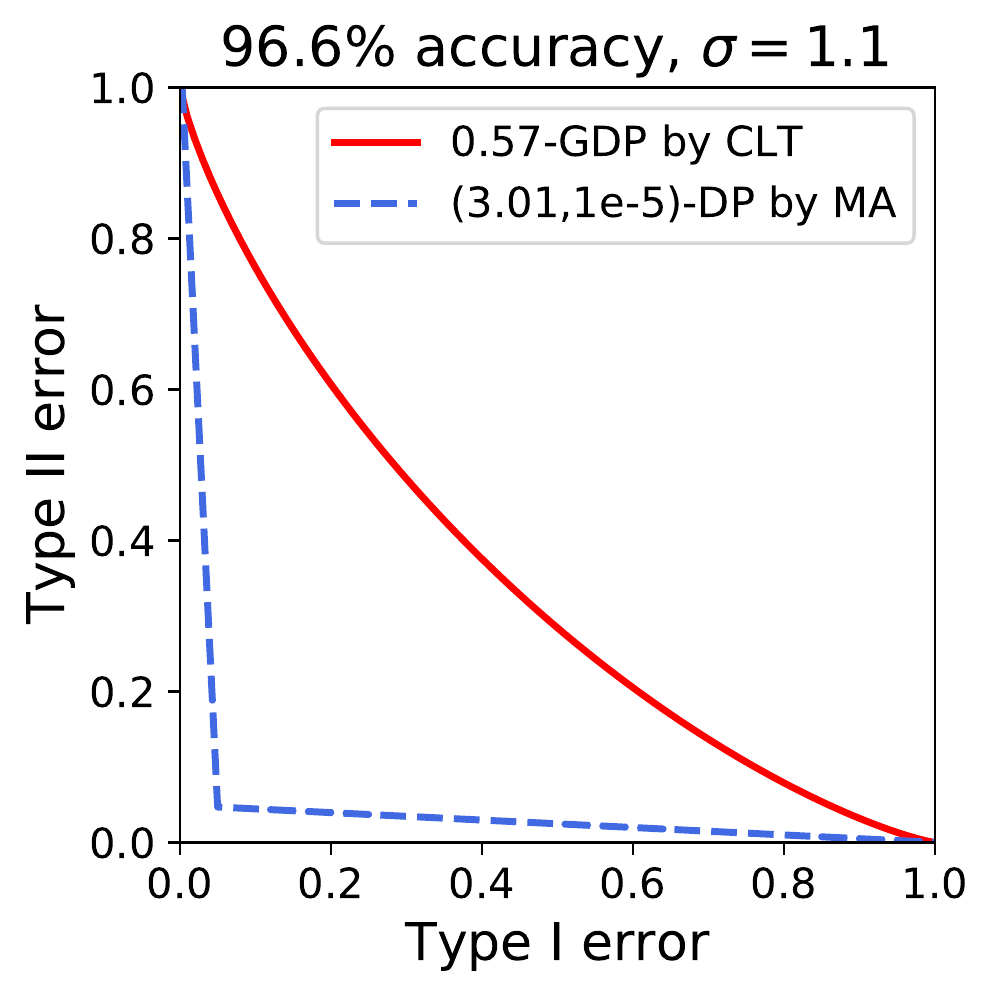}
\includegraphics[width=5cm]{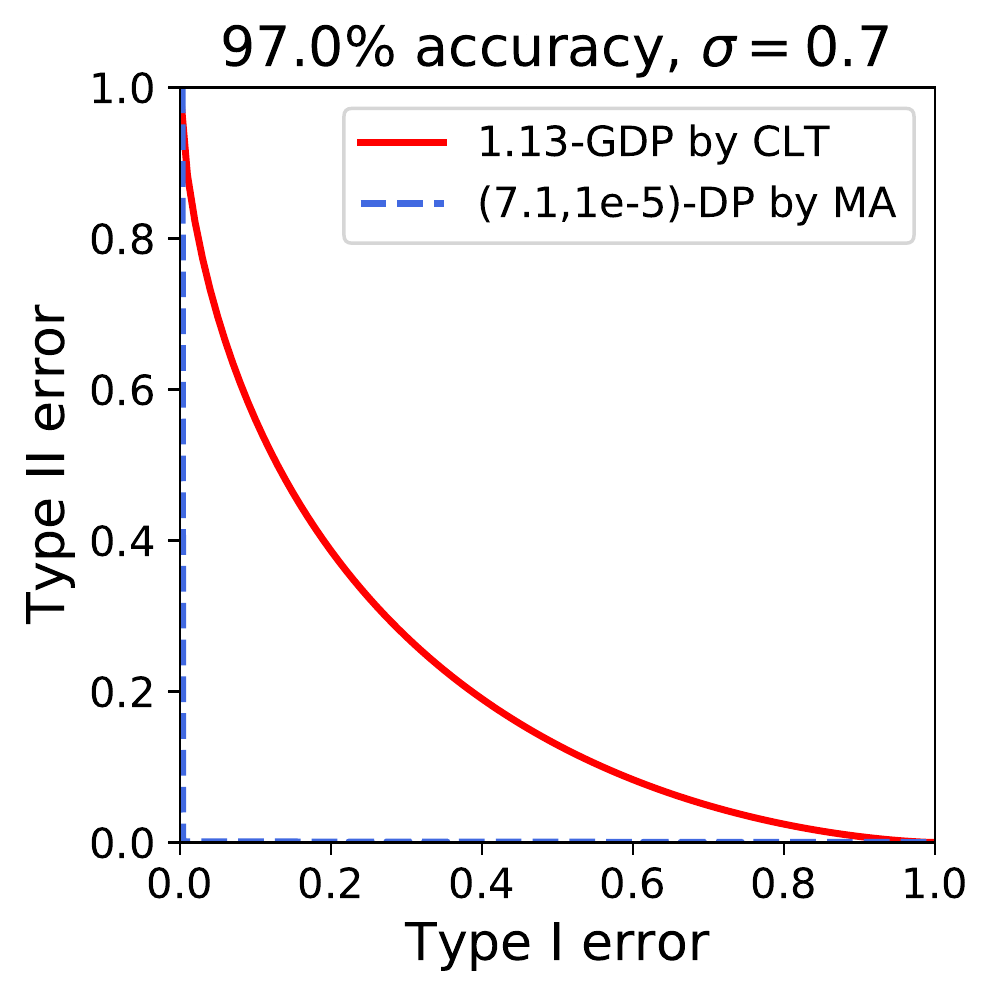}
\\
\includegraphics[width=5cm]{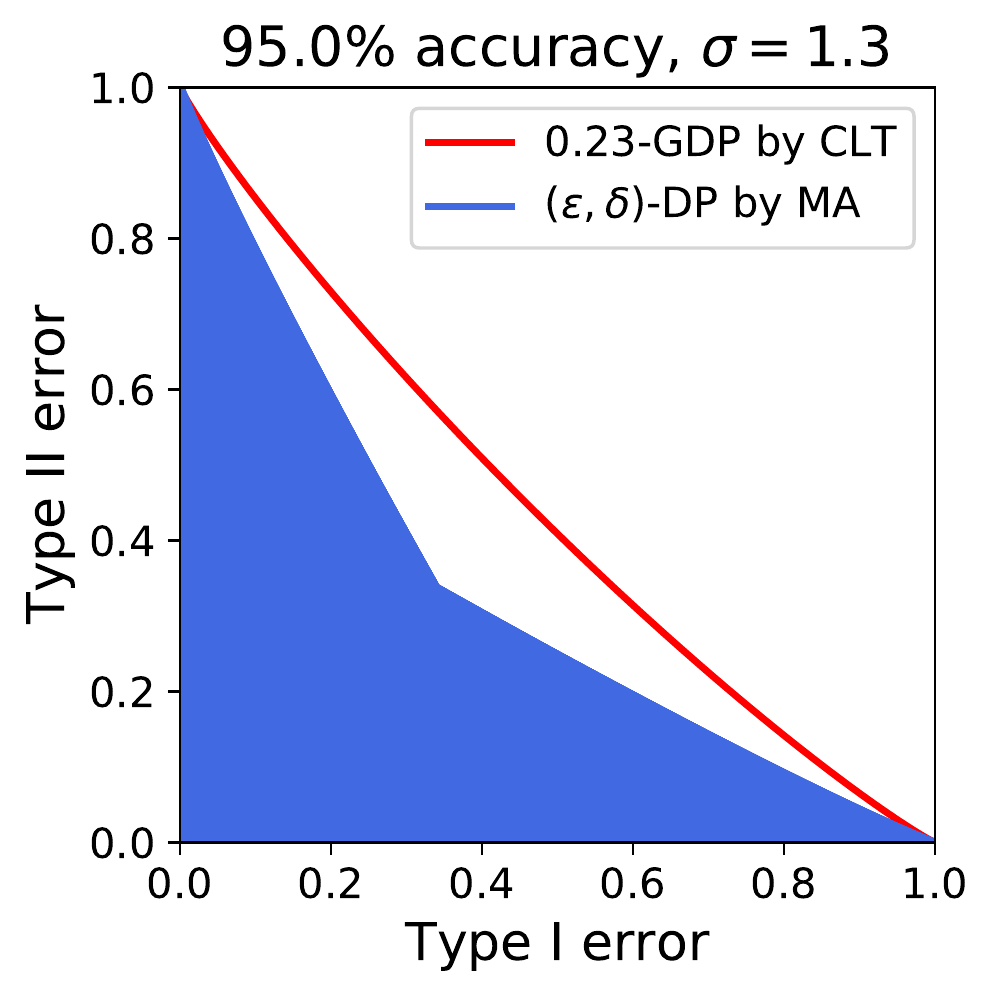}
\includegraphics[width=5cm]{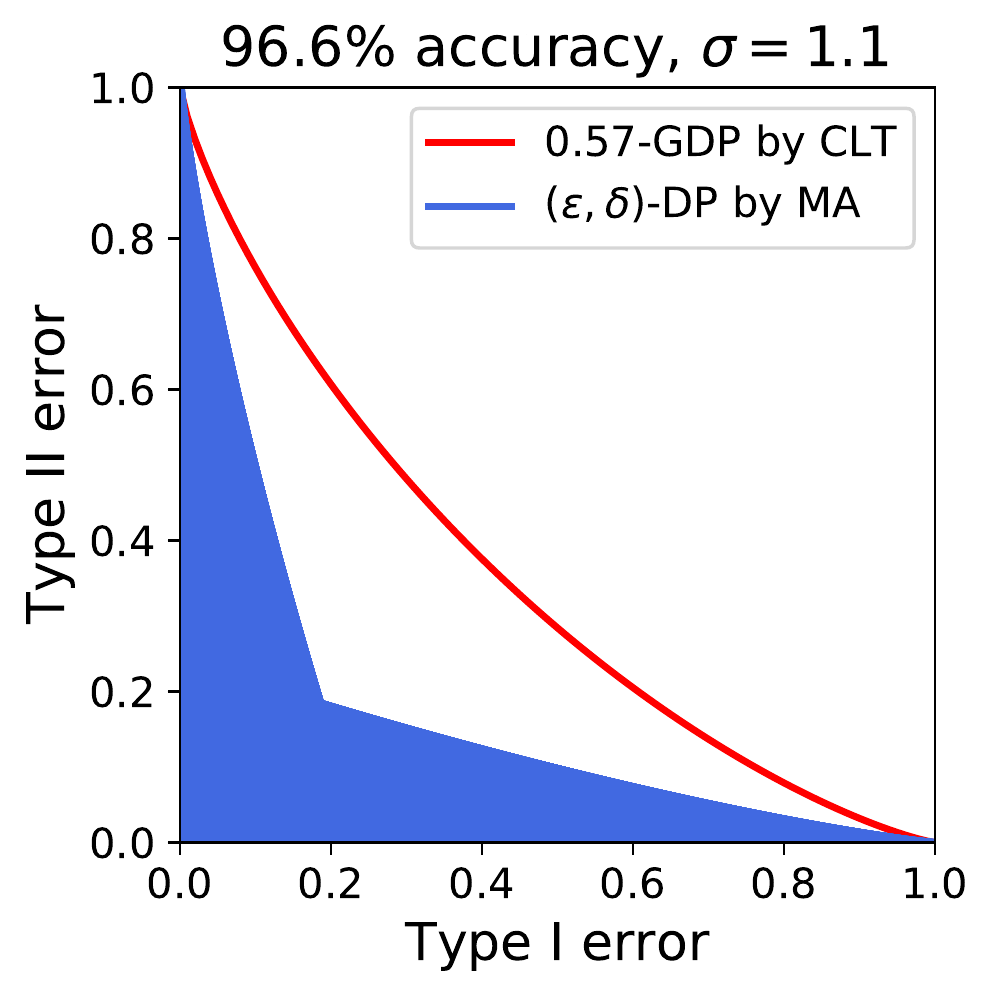}
\includegraphics[width=5cm]{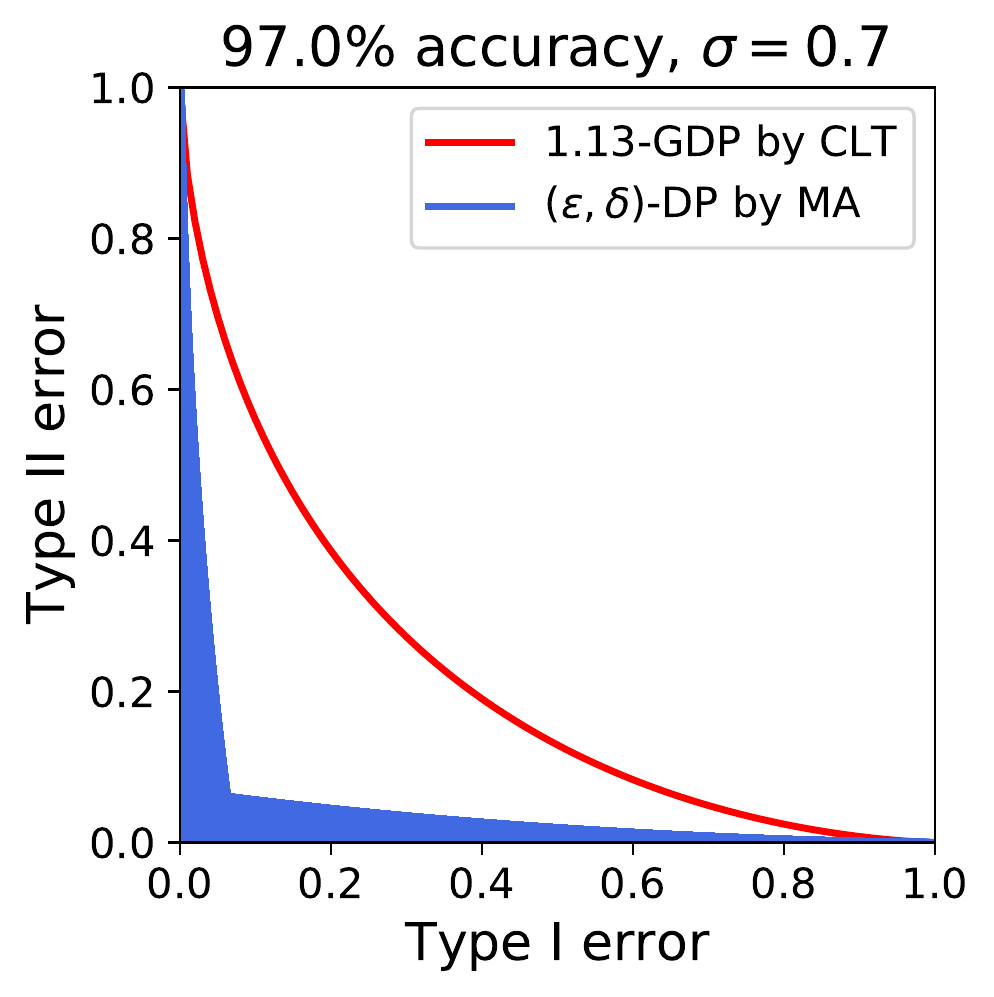}
\\
\includegraphics[width=5cm]{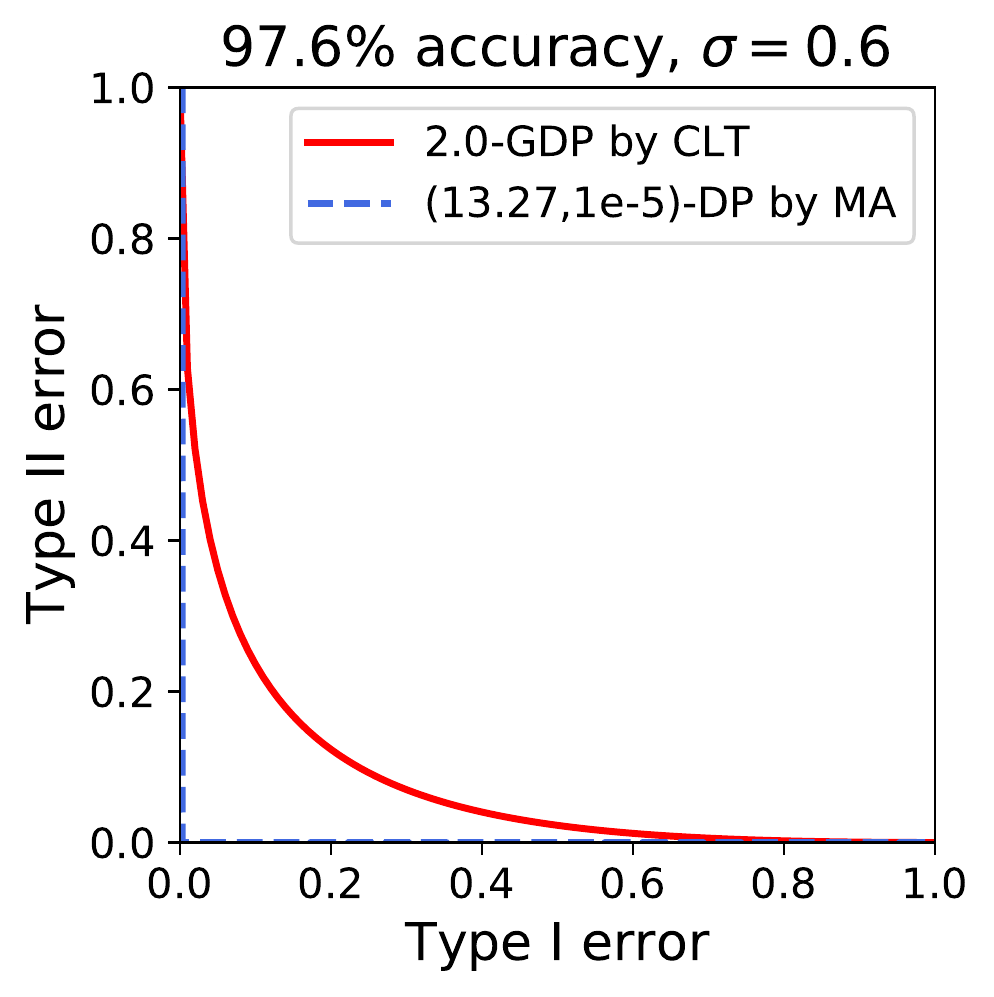}
\includegraphics[width=5cm]{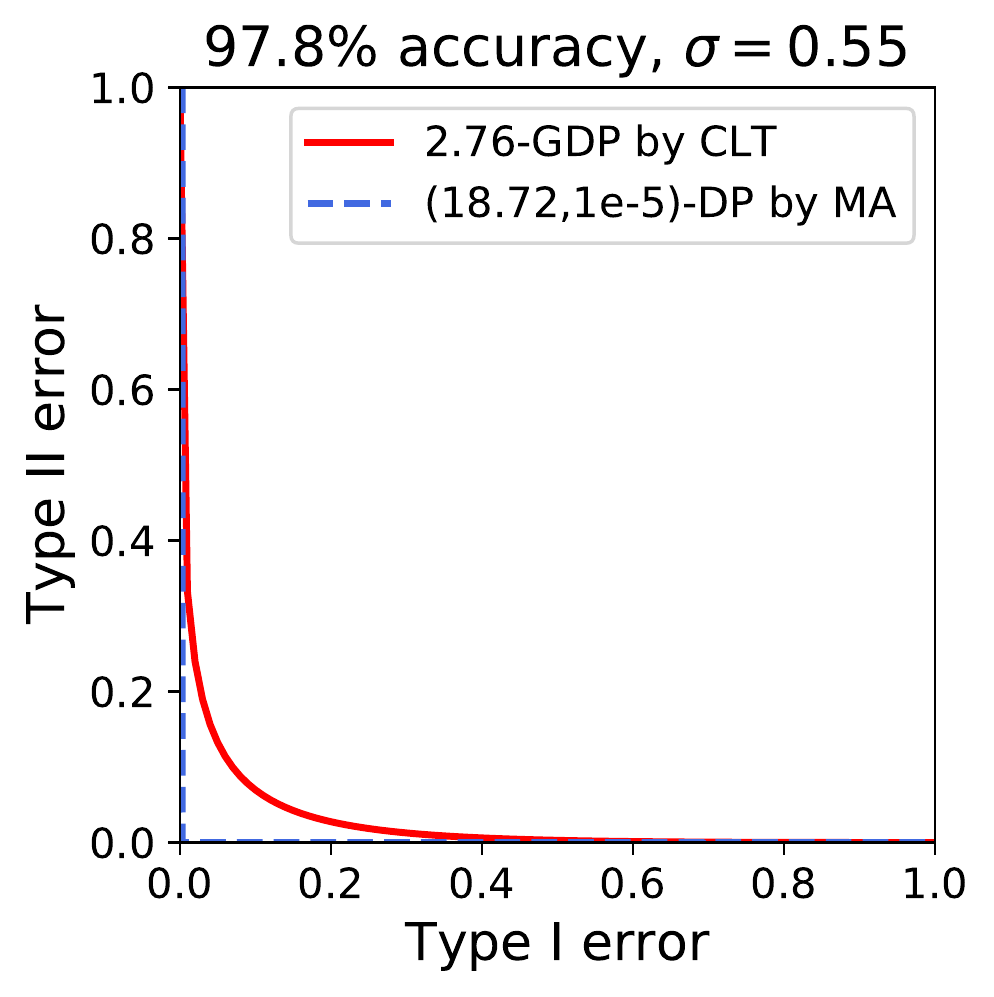}
\includegraphics[width=5cm]{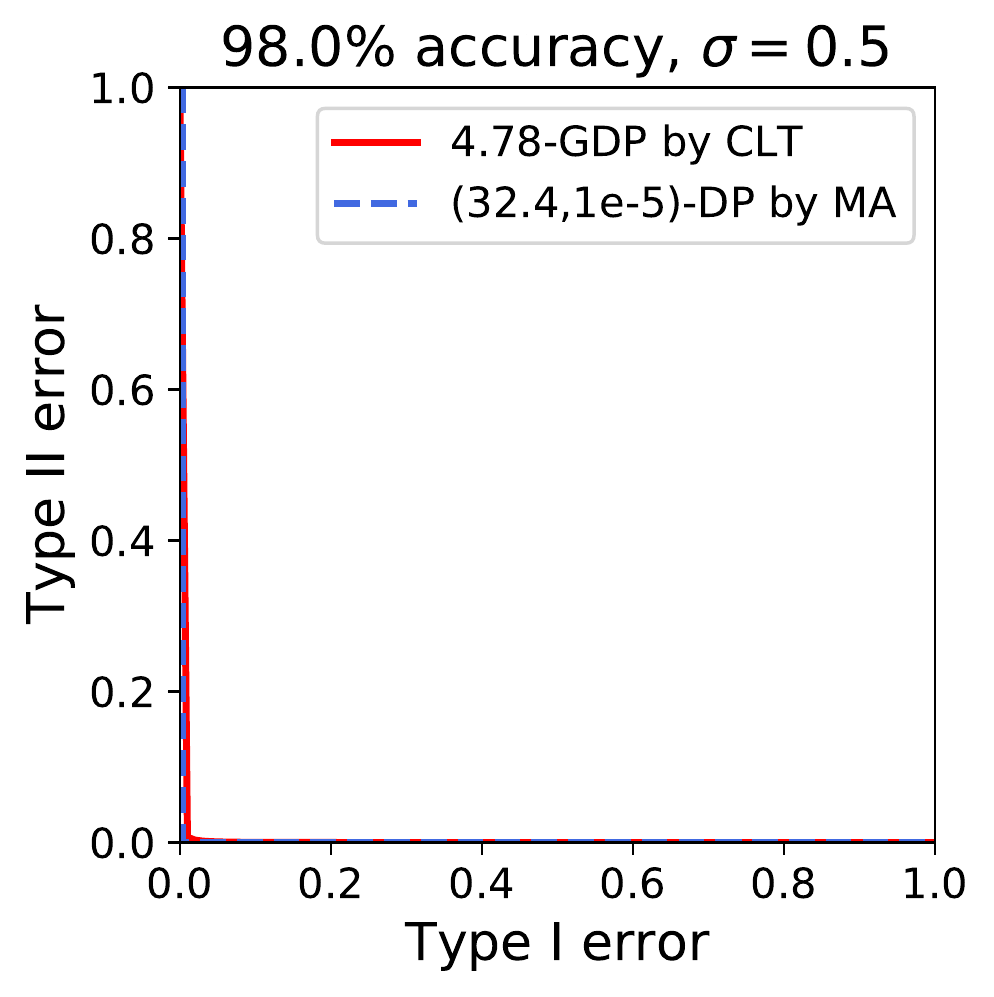}
\caption{Comparisons between the two ways of privacy analysis on MNIST in terms of the trade-off between type I and type II errors, in the same setting as \Cref{tab:tradeoff}. The plots are different from Figure 7 in \cite{dong2019gaussian}. The $(\epsilon, \delta)$-DP guarantees are plotted according to \eqref{eq:e_d_f}. The blue regions in the plots from the second row correspond to all pairs of $(\epsilon,\delta)$ computed by \texttt{MA}. The blue regions are not noticeable in the third row.}
\label{fig:tradeoff}
\end{figure}

For all experiments described in \Cref{tab:tradeoff}, \Cref{fig:tradeoff} illustrates the privacy bounds given by the \texttt{CLT} approach and the moments accountant both in terms of trade-off functions. The six plots in the first and third rows are with respect to $\delta = 10^{-5}$, from which the $f$-DP framework is seen to provide an analyst with substantial improvements in the privacy bounds. Note that the first row in \Cref{fig:tradeoff} corresponds to the first three rows in \Cref{tab:tradeoff}, and the third row in \Cref{fig:tradeoff} corresponds to the last three rows in \Cref{tab:tradeoff}. For the model corresponding to 96.6\% test accuracy, concretely, the \textit{minimum} sum of type I and type II errors in the sense of hypothesis testing is (at least) $77.6\%$ by the \texttt{CLT} approach, whereas it is merely (at least) $9.4\%$ by the moments accountant. For completeness, we show the optimal trade-off functions over all pairs of $\epsilon, \delta$ given by the moments accountant in the middle row. The gaps between the two approaches exist, as predicted by \Cref{thm:f_compare}, and remain significant.

Next, we extend our experiments to other datasets to further test $f$-DP for training private neural networks. The experiments compare private models under the privacy budget $\mu \leq 2$ to their non-private counterparts and some popular baseline methods. For simplicity,  we focus on shallow neural networks and leave the investigation of complex architectures for future research.


\paragraph{Adult income.} Originally from the UCI repository \cite{Dua:2019}, the Adult income dataset has been preprocessed into the LIBSVM format~\cite{chang2011libsvm}. This dataset contains 32,561 examples, each of which has 123 features and a label indicating whether the individual's annual income is more than \$50,000 or not. We randomly choose $10\%$ of the examples as the test set (3,256 examples) and use the remaining 29,305 examples as the training set.

Our model is a single-layer multi-perceptron with 16 neurons and the ReLU activation. We set $\sigma=0.55$, $p=256/29305$, $\eta=0.15, R=1$, and use \texttt{NoisySGD} as our optimizer. The results displayed in \Cref{tab:adult} show that our private model achieves comparable performance to the baselines in the MLC++ library \cite{kohavi1996data} in terms of test accuracy. 


\begin{table}[!htb]
	\centering
	\begin{tabular}{c|c|c|c|c|c}
		\toprule 
		Models& Epochs & Test accuracy (\%) &	\texttt{CLT} $\mu$&	\texttt{CLT} $\epsilon$& \texttt{MA} $\epsilon$\\
		\midrule 
		\midrule 
		private networks&18&84.0&2.03&10.20& 14.70\\
		 non-private networks &20&84.5&---&---&---\\
		$k$NN \cite{cover1967nearest}&---&79.7&---&---&--- \\
		naive Bayes&---&83.9&---&---&--- \\
		voted ID3 \cite{quinlan1986induction}&---&84.4&---&---&--- \\
		C4.5 \cite{quinlan2014c4}&---&84.5&---&---&--- \\
		\bottomrule 
	\end{tabular}
	\caption{Results for \texttt{NoisySGD} on the Adult income dataset. The $\epsilon$ parameters are with respect to $\delta=10^{-5}$.}
	\label{tab:adult}
\end{table}

\paragraph{IMDb.} We use the IMDb movie review dataset~\cite{maas-EtAl:2011:ACL-HLT2011} for binary sentiment classification (positive or negative reviews). The dataset contains 25,000 training and 25,000 test examples. In our experiments, we prepocess the dataset by only including the top 10,000 frequently used words and discard the rest. Next, we set every example to have 256 words by truncating the length or filling with zeros if necessary.


In our neural networks, the input is first embedded into 16 units and then is passed through a global average pooling. The intermediate output is fed to a fully-connected layer with 16 neurons, followed by a ReLU layer. We set $\sigma=0.56$, $p=512/25000$, $\eta=0.02, R=1$, and use \texttt{NoisyAdam} as our optimizer, which is observed to converge much faster than \texttt{NoisySGD} in this training task. We use the (non-private) two-layer LSTM RNN model in the Tensorflow tutorials~\cite{imdb} as a baseline model. \Cref{tab:imdb} reports the experimental results. Notably, the private neural networks perform comparably to the baseline model, at the cost of only one percent drop in test accuracy compared to the non-private counterpart.


\begin{table}[!htb]
	\centering
	\begin{tabular}{c|c|c|c|c|c}
		\toprule 
		Models & Epochs & Test accuracy (\%)& \texttt{CLT} $\mu$& \texttt{CLT} $\epsilon$& \texttt{MA} $\epsilon$\\
		\midrule 
		\midrule 
		private networks &9& 83.8&2.07&10.43& 15.24 \\
		non-private networks &20& 84.7& ---& ---& --- \\
		LSTM-RNN \cite{hochreiter1997long}&10&85.4& ---& ---& ---\\
		\bottomrule 
	\end{tabular}
	\caption{Results for \texttt{NoisyAdam} on the IMDB dataset, with $\delta=10^{-5}$ used in the privacy analyses.}
	\label{tab:imdb}
\end{table}

\paragraph{MovieLens.} The MovieLens movie rating dataset \cite{harper2016movielens} is a benchmark dataset for recommendation tasks. Our experiments consider the MovieLens 1M dataset, which contains 1,000,209 movie ratings from 1 star to 5 stars. In total, there are 6,040 users who rated 3,706 different movies. For this multi-class classification problem, the root mean squared error (RMSE) is chosen as the performance measure. It is worthwhile to mention that, as each user only watched a small fraction of all the movies, most (user, movie) pairs correspond to missing ratings. We randomly sample 20\% of the examples as the test set and take the remainder as the training set.


Our model is a simplified version of the neural collaborative filtering in \cite{he2017neural}. The network architecture consists of two branches. The left branch applies generalized matrix factorization to embed the users and movies using five latent factors. The output of the user embedding is multiplied by the item embedding. In the right branch, we use 10 latent factors for embedding. The embedding from both branches are then concatenated, which is fed to a fully-connected output layer. We set $\sigma=0.6, p=1/80, \eta=0.01$, and $R=5$ in \texttt{NoisyAdam}.

\begin{table}[!htb]
	\centering
	\begin{tabular}{c|c|c|c|c|c}
		\toprule 
		Models & Epochs & RMSE & \texttt{CLT} $\mu$ & \texttt{CLT}  $\epsilon$&	\texttt{MA} $\epsilon$\\
		\midrule 
		\midrule 
		 private networks &20&0.915&1.94&10.61&15.39\\
		 non-private networks &20&0.893&---&---&---\\
		SVD&---&0.873&---&---&---\\
		NMF&---&0.916&---&--- &---\\
		user-based CF \cite{sarwar2001item}&---&0.923&---&---&--- \\
		global average&---&1.117&---&---&--- \\
		\bottomrule 
	\end{tabular}
	\caption{Results for \texttt{NoisyAdam} on the MovieLens 1M dataset, with $\delta=10^{-6}$ used in the privacy analyses. CF stands for collaborative filtering.}
	\label{tab:movielens}
\end{table}

\Cref{tab:movielens} presents the numerical results of our neural networks as well as baseline models in the Suprise library~\cite{surprise} in their default settings. The difference in RMSE between the non-private networks and the private one is relatively large for the MovieLens 1M dataset. Nevertheless, the private model still outperforms many popular non-private models, including the user-based collaborative filtering and nonnegative matrix factorization. 


\subsection{The $(\epsilon, \delta)$-DP Perspective}
\label{sec:clt-impr-perf}



While we hope that the $f$-DP perspective has been conclusively demonstrated to be advantageous, this section shows that the \texttt{CLT} approach continues to bring considerable benefits even in terms of $(\epsilon, \delta)$-DP. Specifically, by making use of the comparisons between the \texttt{CLT} approach and the moments accountant in \Cref{sec:conn-with-moments}, we can add less noise to the gradients in \texttt{NoisySGD} and \texttt{NoisyAdam} while achieving the same $(\epsilon, \delta)$-DP guarantees provided by the moments accountant. With less added noise, conceivably, an optimizer would have a higher prediction accuracy.


Figure~\ref{fig:accuracy_boost1} illustrates the experimental results on MNIST. In the top two plots, we set the noise scales to $\sigma=1.3, \tilde\sigma=1.06$, which are both shown to give $(1.34, 10^{-5})$-DP at epoch 20 using the moments accountant and the \texttt{CLT} approach, respectively. The test accuracy associated with the \texttt{CLT} approach is almost always higher than that associated with the moments accountant. In addition, another benefit of taking the \texttt{CLT} approach is that it gives rise to stronger privacy protection before reaching epoch 20, as shown by the right plot. For the bottom plots, although the improvement in test accuracy at the end of training is less significant, the \texttt{CLT} approach leads to much faster convergence at early epochs. To be concrete, the numbers of epochs needed to achieve $95\%, 96\%$, and $97\%$ test accuracy are $18, 26$, and $45$, respectively, for the neural networks with less noise, whereas the numbers of epochs are $23,
33$, and $64$, respectively, using noise level that is computed by the moments accountant. In a similar vein, the moments accountant gives a test accuracy of 92\% for the first time when $\epsilon = 4$ and the \texttt{CLT} approach achieves 96\% under the same privacy budget.


\begin{figure}[!htb]
               \centering
	\includegraphics[width=8cm]{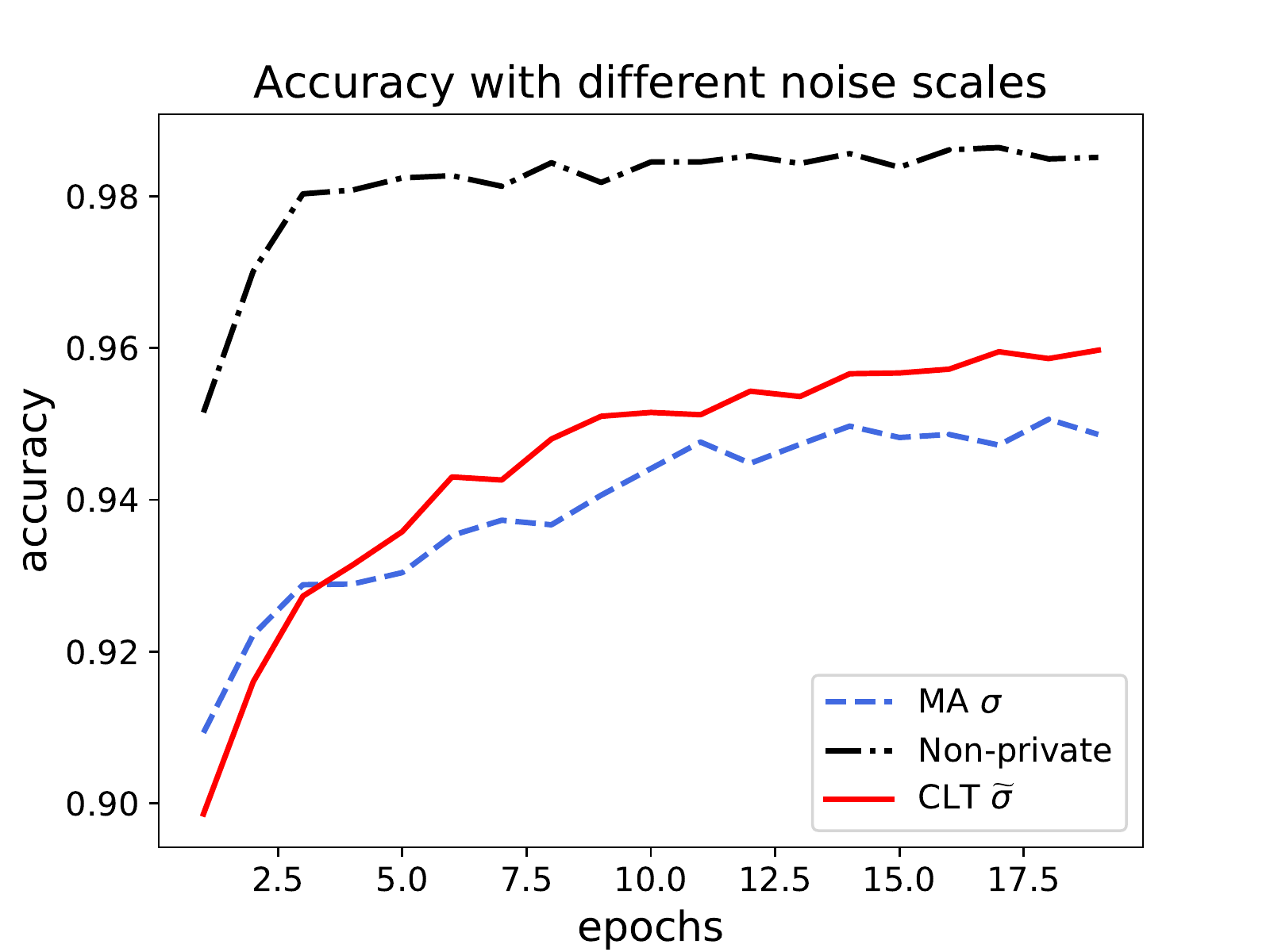}
	\hspace{-0.7cm}
	\includegraphics[width=8cm]{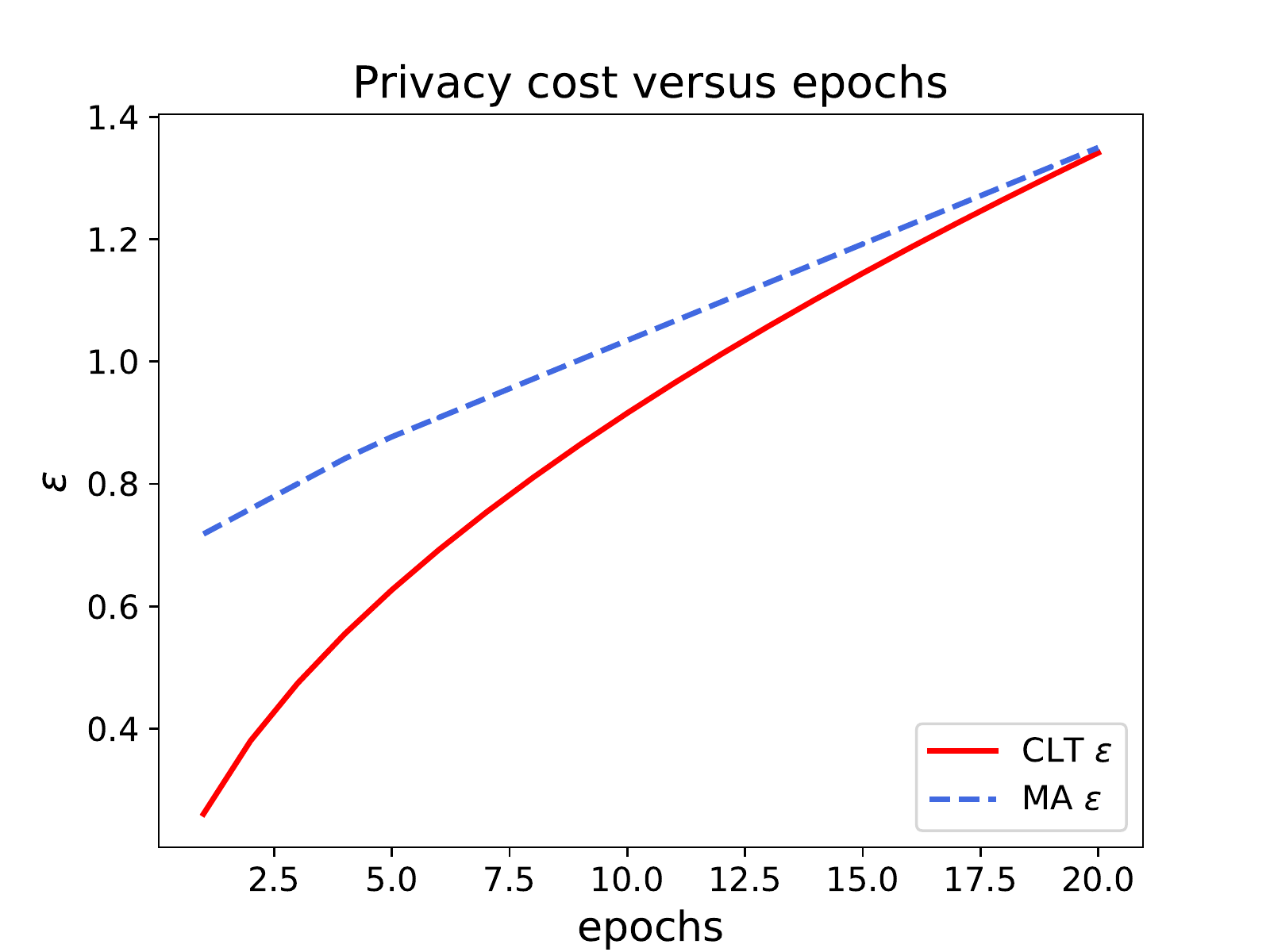}

	\includegraphics[width=8cm]{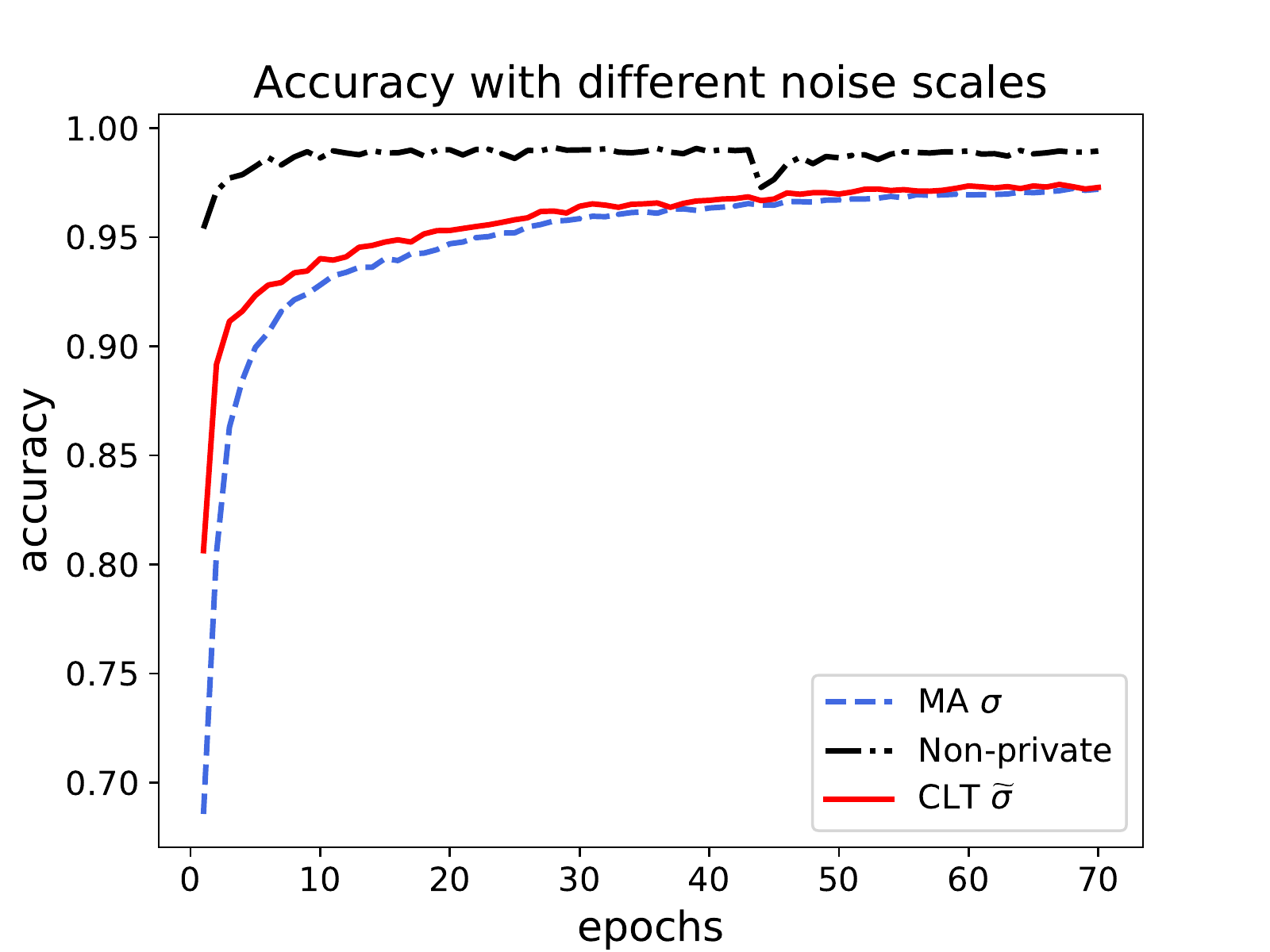}
	\hspace{-0.7cm}
	\includegraphics[width=8cm]{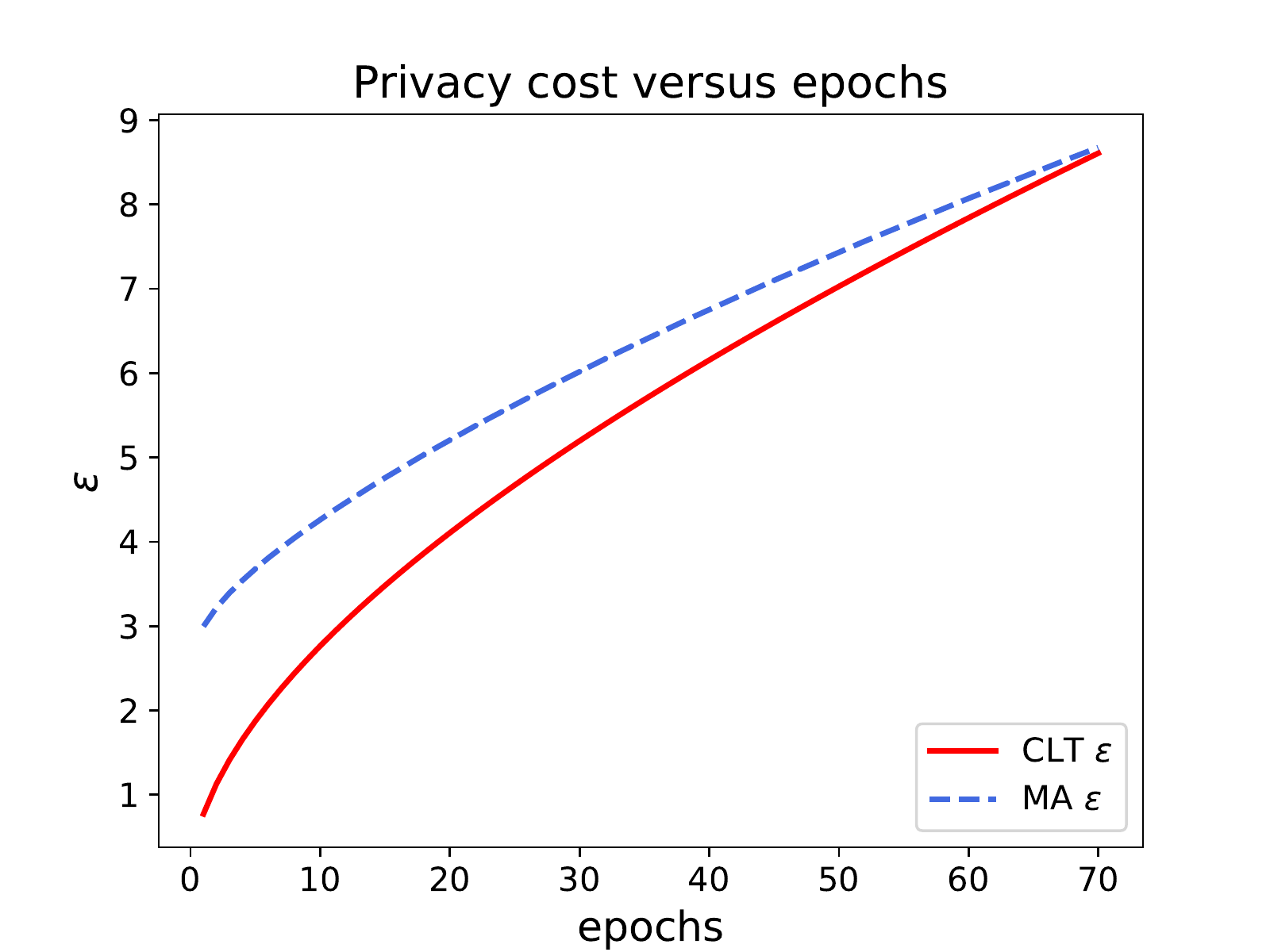}
	\caption{Experimental results from one run of \texttt{NoisySGD} on MNIST with different noise scales but the same $(\epsilon, \delta)$-DP guarantees. The top plots use $p=256/60000, \eta=0.15, R=1.5$, and $\sigma = 1.3, \tilde\sigma=1.06$. The \texttt{CLT} approach with $\tilde\sigma=1.06$ and the moments accountant with $\sigma=1.3$ give $(1.34, 10^{-5})$-DP at the 20th epoch ($\mu_{\mathtt{CLT}}=0.35$). The bottom plots use the same parameters except for $\sigma=0.7, \tilde\sigma=0.638$, and $\eta = 0.15$. Both approaches give $(8.68, 10^{-5})$-DP at epoch 70 ($\mu_{\mathtt{CLT}}=1.78$). The right plots show the privacy loss during the training process in terms of the $\epsilon$ spending with respect to $\delta=10^{-5}$.}
		\label{fig:accuracy_boost1}
\end{figure}

\section{Discussion}
\label{sec:discussion}


In this paper, we have showcased the use of $f$-DP, a very recently proposed privacy definition, for training private deep learning models using SGD or Adam. Owing to its strength in handling composition and subsampling and the powerful privacy central limit theorem, the $f$-DP framework allows for a closed-form privacy bound that is sharper than the one given by the moments accountant in the $(\epsilon, \delta)$-DP framework. By numerical experiments, we show that the trained neural networks can be quite private from the $f$-DP viewpoint (for instance, $1.13$-GDP\footnote{This means that undermining the privacy guarantee is harder than or of the same hardness as testing $H_0: \mu=0$ against $H_1: \mu=1.13$ based on the observation $\mu + \N(0,1)$.}) but are \textit{not} in the $(\epsilon, \delta)$-DP sense due to over conservative privacy bounds (for instance, $(7.10, 10^{-5})$-DP) computed in the $(\epsilon, \delta)$-DP framework. This in turn suggests that one can add less noise during
the training process while having the same privacy guarantees as using the moments accountant, thereby improving model utility.

We conclude this paper by offering several directions for future research. As the first direction, we may consider using time-dependent noise scales and learning rates in \texttt{NoisySGD} and \texttt{NoisyAdam} for a better tradeoff between privacy loss and utility in the $f$-DP framework. Note that \cite{lee2018concentrated} has made considerable progress using concentrated differential privacy along this line. More generally, a straightforward but interesting problem is to extend this work to complex neural network architectures with a variety of optimization strategies. For example, can we develop some guidelines for choosing an optimizer among \texttt{NoisySGD}, \texttt{NoisyAdam}, and others for a given classification problem under some privacy constraint? Empirically, deep learning models are very sensitive to hyperparameters such as mini-batch size in terms of test accuracy. Therefore, from a practical standpoint, it would be of great importance to incorporate
hyperparameter tuning into the $f$-DP framework~\cite{gupta2010differentially}. Inspired by \cite{lecuyer2019certified}, another interesting direction is to explore the possible relationship between $f$-DP guarantees and adversarial robustness of neural networks. Given $f$-DP's good interpretability and powerful toolbox, it is worthwhile investigating whether, from a broad perspective, its superiority over earlier differential privacy relaxations would hold in general private statistical and machine learning tasks. We look forward to more research efforts to further the theory and extend the use of $f$-DP.


\subsection*{Acknowledgments}
We are grateful to David Durfee, Ryan Rogers, Aaron Roth, and Qinqing Zheng for stimulating discussions in the early stages of this work. We would also like to thank two anonymous referees for their constructive comments that improved the presentation of the paper. This work was supported in part by NSF through CAREER DMS-1847415, CCF-1763314, and CCF-1934876, the Wharton Dean's Research Fund, and NIH through R01GM124111 and RF1AG063481.


{\small
\bibliographystyle{abbrvnat}
\bibliography{ref}
}

\clearpage
\appendix
\addcontentsline{toc}{section}{Appendices}
\newcommand{\MA}{\mathtt{MA}}
\newcommand{\nMA}{\mathtt{nMA}}
\newcommand{\Dual}{\mathtt{Dual}}
\newcommand{\GDP}{\mathtt{GDP}}
\newcommand{\GM}{\mathrm{GM}}

\section{Omitted Details in \Cref{sec:preliminaries}} 
\label{sec:numerical}


We present \Cref{eq:subsample} as the following proposition, which is given in \Cref{sec:preliminaries} but not in the foundational work \cite{dong2019gaussian}.


\begin{proposition}\label{prop:subsample}
If $M$ is $f$-DP, and $S'=S\cup\{x_0\}$, then
$$
T\big(M\circ \Sample_p(S), M\circ \Sample_p(S')\big)\geqslant pf+(1-p)\Id.
$$
\end{proposition}
\begin{proof}
	We first write the two distributions $M\circ \Sample_p(S)$ and $M\circ \Sample_p(S')$ as mixtures.

	Without loss of generality, we can assume $S=\{x_1,\ldots, x_n\}$ and $S'=\{x_0,x_1,\ldots, x_n\}$. An outcome of the process $\Sample_p$ when applied to $S$ is a bit string $\vec{b} = (b_1,\ldots, b_n)\in \{0,1\}^n$. Bit $b_i$ dependes on whether $x_i$ is selected into the subsample. We use $S_{\vec{b}}\subseteq S$ to denote the subsample determined by $\vec{b}$. When each $b_i$ is sampled from a Bernoulli$(p)$ distribution independently, $S_{\vec{b}}$ can be identified with $\Sample_p(S)$.  Let $\theta_{\vec{b}}$ be the probability that $\vec{b}$ appears. More specifically, if $k$ out of $n$ entries of $\vec{b}$ is one, then $\theta_{\vec{b}}=p^k(1-p)^{n-k}$. With this notation, $M\circ \Sample_p(S)$ can be written as the following mixture:
	\[M\circ \Sample_p(S) = \sum_{\vec{b}\in\{0,1\}^n} \theta_{\vec{b}}\cdot M(S_{\vec{b}}).\]
	Similarly, $M\circ \Sample_p(S)$ can also be written as a mixture, with an additional bit indicating the presence of $x_0$. Alternatively, we can divide the components into two groups: one with $x_0$ present, and the other with $x_0$ absent. Namely,
	\[M\circ \Sample_p(S') = \sum_{\vec{b}\in\{0,1\}^{n}} p\cdot\theta_{\vec{b}}\cdot M(S_{\vec{b}}\cup \{x_0\})+\sum_{\vec{b}\in\{0,1\}^{n}} (1-p)\cdot\theta_{\vec{b}}\cdot M(S_{\vec{b}}).\]
	Note that $S_{\vec{b}}\cup \{x_0\}$ and $S_{\vec{b}}$ are neighbors, i.e.~$M\circ \Sample_p(S')$ is the mixture of neighboring distributions. The following lemma is the perfect tool to deal with it.
	\begin{lemma}\label{lem:mixture}
		Let $I$ be an index set. For all $i\in I$, $P_i$ and $Q_i$ are distributions that reside on a common sample space. $(\theta_i)_{i\in I}$ is a collection of non-negative numbers that sums to 1.
		If $f$ is a trade-off function and $T(P_i,Q_i)\geqslant f$ for all $i$, then 
		\[T\big(\sum \theta_i P_i, (1-p)\sum \theta_i P_i +p\sum \theta_i Q_i\big)\geqslant p f+ (1-p)\Id.\]
	\end{lemma}
	To apply the lemma, let the index be $\vec{b}\in\{0,1\}^n$, $P_i$ be $M(S_{\vec{b}})$ and $Q_i$ be $M(S_{\vec{b}}\cup \{x_0\})$. Condition $T(P_i,Q_i)\geqslant f$ is the consequence of $M$ being $f$-DP. The conclusion simply translates to
	$$T\big(M\circ \Sample_p(S), M\circ \Sample_p(S')\big)\geqslant pf+(1-p)\Id,$$
	which is what we want. The proof is complete.
\end{proof}
\begin{proof}[Proof of \Cref{lem:mixture}]
	Let $P = \sum \theta_i P_i$ and $Q = (1-p)\sum \theta_i P_i +p\sum \theta_i Q_i$. Suppose $\phi$ satisfies $\E_P \phi=\alpha$. That is,
	\[\sum \theta_i \E_{P_i}\phi = \alpha.\]
	It is easy to see that
	\begin{align*}
		\E_Q \phi = (1-p)\alpha+p \sum \theta_i \E_{Q_i}\phi.
	\end{align*}
	We know that $T(P_i,Q_i)\geqslant f$. Hence $\E_{Q_i}\phi\leqslant 1-f(\E_{P_i}\phi)$. So
	\begin{align*}
		\sum \theta_i \E_{Q_i}\phi \leqslant 1-\sum \theta_i f(\E_{P_i}\phi).
	\end{align*}
	Since $f$ is convex, Jensen's inequality implies
	\[\sum \theta_i f(\E_{P_i}\phi)\geqslant f(\sum \theta_i \E_{P_i}\phi)=f(\alpha).\]
\end{proof}


Next we use a figure to justify the claim we made in \Cref{sub:composition} that ``CLT approximation works well for SGD''. Recall that we argued in \Cref{sec:implementation} that \Cref{alg:dpsgd1,alg:dpsgd2} are $\min\{f,f^{-1}\}^{**}$-DP where
$$
f = \big(pG_{1/\sigma}+(1-p)\Id\big)^{\otimes T}.
$$
This function converges to $G_\mu$ with $\mu = \nu\sqrt{\e^{1/\sigma^2}-1}$ as $T\to\infty$ provided $p\sqrt{T}\to\nu$. In the following figure, we numerically compute $f$ (blue dashed) and compare it with the predicted limit $G_\mu$ (red solid). More specifically, the configuration is designed to illustrate the fast convergence in the setting of the second line of \Cref{tab:tradeoff}, i.e. noise scale $\sigma =1.1$, final GDP parameter $\mu=0.57$ and test accuracy 96.6\%. Originally the algorithm runs 60 epochs, i.e. $\approx$ 14k iterations. To best illustrate that convergence appears in early stage, the numerical evaluation uses a much smaller $T_{\textup{numeric}}=234$, i.e. only \emph{one} epoch. In order to make the final limit consistent, we also enlarge the sample probability to $p_{\textup{numeric}}$ so that $p_{\textup{numeric}}\cdot \sqrt{T_{\textup{numeric}}}$ remain the same.
\begin{figure}[!htb]
	\centering
	\includegraphics[width=0.5\textwidth]{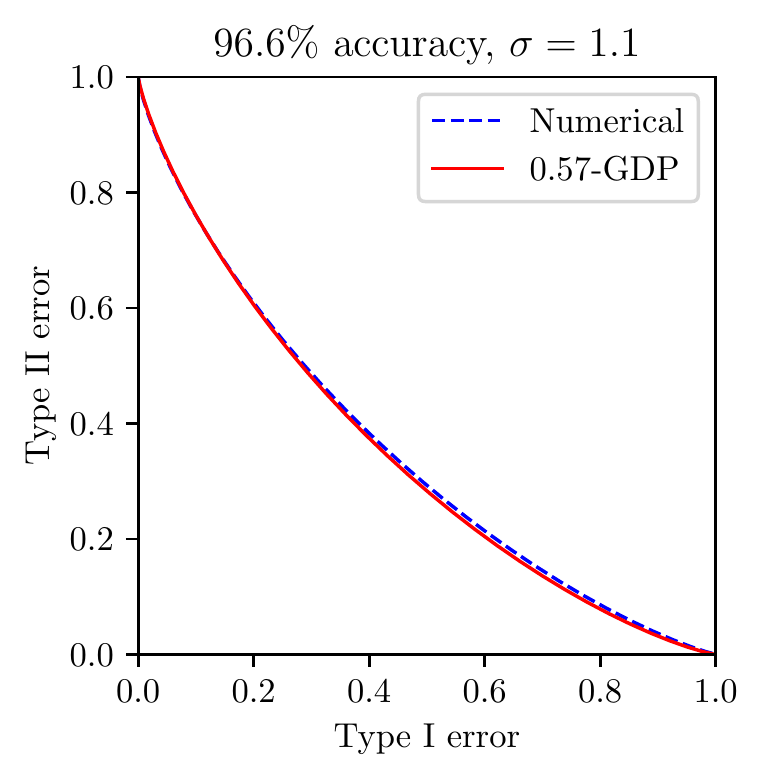}
	\caption{$\big(pG_{1/\sigma}+(1-p)\Id\big)^{\otimes T}$ (blue dashed) is numerically computed and compared with the GDP limit (red solid) predicted by CLT. The two are almost identical at merely epoch one.}
	\label{fig:necessary}
\end{figure}

We have to remark that when $\sigma$ is small, $\mu = \nu\sqrt{\e^{1/\sigma^2}-1}$ gets large and yields challenges in the numerical computation of $\big(pG_{1/\sigma}+(1-p)\Id\big)^{\otimes T}$. We leave rigorous and complete study to future work.

\section{Omitted Details in \Cref{sec:implementation}}
\label{sec:omitt-deta-crefs}
\subsection{Privacy Property of \Cref{alg:dpsgd1,alg:dpsgd2}} 
\label{sub:privacy_property_of_alg:dpsgd1,alg:dpsgd2}

\begin{theorem} \label{thm:}
	\Cref{alg:dpsgd1,alg:dpsgd2} are both $\min\{f,f^{-1}\}^{**}$-DP with $f = \big(pG_{1/\sigma}+(1-p)\Id\big)^{\otimes T}$. 
\end{theorem}
\begin{proof}
	The proof is mostly done in the main text, except the composition step. Let $V$ be the vector space that all $\theta_t$ live in and $\widetilde{M}=M\circ \Sample_p: X^n\times V\to V$ be the gradient update. We have already proved (using \Cref{prop:subsample}) that for both \Cref{alg:dpsgd1,alg:dpsgd2}, if $S'=S\cup\{x_0\}$, then $\widetilde{M}$ satisfies
	$$T\big(M(S),M(S')\big)\geqslant f_p:=pG_{1/\sigma}+(1-p)\Id.$$
	Note that we cannot say $M$ is $f_p$-DP because $T\big(M(S'),M(S)\big)$ is not necessarily lower bounded by $f_p$. So we need a more specific composition theorem than stated in \cite{dong2019gaussian}.
	\begin{theorem}[Refined Composition] \label{thm:}
		Suppose $M_1:X\to Y, M_2:X\times Y \to Z$ satisfy the following conditions for any $S,S'$ such that $S'=S\cup\{x_0\}$:
		\begin{enumerate}
			\item 
	$T\big(M_1(S),M(S')\big)\geqslant f$;
	\item
	$T\big(M_2(S,y),M_2(S',y)\big)\geqslant g$ for any $y\in Y$.
		\end{enumerate}
		Then the composition $M_2\circ M_1: X\to Y\times Z$ satisfies
	$$T\big(M_2\circ M_1(S),M_2\circ M_1(S')\big)\geqslant f\otimes g$$
	for any $S,S'$ such that $S'=S\cup\{x_0\}$.
	\end{theorem}
	The theorem can be identically proved as Theorem 3.2 in \cite{dong2019gaussian}.

	Taking \Cref{alg:dpsgd1} as an example, since
	\begin{align*}
		\mathrm{\texttt{NoisySGD}}:X^n&\to V\times V\times\cdots\times V\\
		S&\mapsto(\theta_1,\theta_2,\ldots, \theta_T)
	\end{align*}
is simply the composition of $T$ copies of $\widetilde{M}$, the above composition theorem implies that
\[T\big(\mathrm{\texttt{NoisySGD}}(S),\mathrm{\texttt{NoisySGD}}(S')\big)\geqslant\big(pG_{1/\sigma}+(1-p)\Id\big)^{\otimes T}=f.\]
Moreover, 
$T\big(\mathrm{\texttt{NoisySGD}}(S),\mathrm{\texttt{NoisySGD}}(S')\big)\geqslant f^{-1}$. The two inequality let us conclude that any trade-off function of neighboring distributions must be lower bounded by at least one of $f$ and $f^{-1}$, hence $\min\{f,f^{-1}\}$, hence $\min\{f,f^{-1}\}^{**}$. In other words, \texttt{NoisySGD} is $\min\{f,f^{-1}\}^{**}$-DP.

For \texttt{NoisyAdam}, we argued that its privacy property is the same as \texttt{NoisySGD} in each iteration, so the above argument also applies, and we have the same conclusion.
\end{proof}

\subsection{Justifying CLT for \Cref{alg:dpsgd1,alg:dpsgd2}} 
\label{sub:unjustified_claim_in_sec:noisy-sgd}

\newcommand{\kl}{\mathrm{kl}}
The main purpose of this section is to show the following theorem
\begin{theorem} \label{thm:}
Suppose $p$ depends on $T$ and $p\sqrt{T}\to \nu$. Then we have the following uniform convergence as $T\to\infty$
	$$
	\big(pG_{1/\sigma}+(1-p)\Id\big)^{\otimes T} = G_\mu,
	$$
	where $\mu = \nu\cdot \sqrt{T(\e^{1/\sigma^2}-1)}$.
\end{theorem}
This theorem is the corollary of the following more general CLT on composition of subsample mechanisms and \Cref{lem:chi_G} below.
\begin{theorem}\label{thm:chi_square}
	Suppose $f$ is a trade-off function such that (1) $f(0)=1$, (2) $f(x)>0$, for all $x<1$ and (3) $\int_0^1(f'(x)+1)^4\diff x<+\infty$. Let $f_p = pf+(1-p)\Id$ as usual. Furthermore, assume $p\sqrt{T} \to \nu$ as $T\to\infty$ for some constant $\nu > 0$. Then we have the uniform convergence
	$$f_p^{\otimes T}\to G_{\nu\sqrt{\chi^2(f)}}$$
as $T \to \infty$, where $\chi^2(f) = \int_0^1 \big(f'(x)\big)^2\diff x-1$.
\end{theorem}
\begin{lemma}\label{lem:chi_G}
We have \begin{align*}
	\chi^2(G_{1/\sigma}) = \e^{1/\sigma^2}-1.
\end{align*}
\end{lemma}
In order to prove \Cref{thm:chi_square}, we need an even more general CLT.
The first privacy CLT was introduced in \cite{dong2019gaussian}. However, that version is valid only when each component trade-off function is symmetric, which is not true for $pG_{1/\sigma}+(1-p)\Id$. In order to state the general CLT that applies to asymmetric trade-off functions, we need to introduce the following functionals:
\begin{align*}
	\kl(f) &:= -\int_0^1\log |f'(x)|\diff x\\
	\tilde{\kl}(f) &:= \int_0^1|f'(x)|\log |f'(x)|\diff x\\
	\kappa_2(f)&:=\int_0^1\log^2 |f'(x)| \diff x\\
	\tilde{\kappa}_2(f)&:=\int_0^1|f'(x)|\log^2 |f'(x)| \diff x\\
	\kappa_3(f)&:=\int_0^1\big|\log |f'(x)|\big|^3\diff x\\
	\tilde{\kappa}_3(f)&:=\int_0^1|f'(x)|\cdot\big|\log |f'(x)|\big|^3\diff x.
\end{align*}

\begin{theorem}\label{thm:CLT}
	Let $\{f_{ni}: 1\leqslant i \leqslant n\}_{n=1}^{\infty}$ be a triangular array of (possibly asymmetric) trade-off functions and assume the following limits for some constants $K \ge 0$ and $s > 0$ as $n \to \infty$:
	\begin{enumerate}
		\item[\textup{1.}] $\sum_{i=1}^n \kl(f_{ni})+\tilde{\kl}(f_{ni})\to K;$
		\item[\textup{2.}] $\max_{1\leqslant i\leqslant n} \kl(f_{ni}) \to 0, \quad \max_{1\leqslant i\leqslant n} \tilde{\kl}(f_{ni}) \to 0;$
		\item[\textup{3.}] $\sum_{i=1}^n \kappa_2(f_{ni})\to s^2, \quad\sum_{i=1}^n \tilde{\kappa}_2(f_{ni})\to s^2;$
		\item[\textup{4.}] $\sum_{i=1}^n \kappa_3(f_{ni})\to 0,\quad \sum_{i=1}^n \tilde{\kappa}_3(f_{ni})\to 0$.
	\end{enumerate}
Then, we have
	$$\lim_{n\to \infty} f_{n1}\otimes f_{n2} \otimes \cdots \otimes f_{nn} (\alpha) = G_{K/s}(\alpha)$$
uniformly for all $\alpha \in [0,1]$.
\end{theorem}
Proof of this theorem exactly mimics that of Theorem 3.5 in \cite{dong2019gaussian}, which we omit here for its length and tediousness.

Next, we apply the asymmetric CLT to $\big(pf+(1-p)\Id\big)^{\otimes T}$ and prove \Cref{thm:chi_square}. We start by collecting the necessary expressions into the following lemma. All of them are straightforward.
\begin{lemma} \label{lem:g}
Let $g(x) = -f'(x)-1 = |f'(x)|-1$. Then
	\begin{align*}
		{\kl}(f_p) &= - \int_0^1\log(1+pg(x))\diff x\\
		\tilde{\kl}(f_p) &=\int_0^1(1+pg(x))\log(1+pg(x))\diff x\\
		\kappa_2(f_p)&= \int_0^{1}\big[\log \big(1+pg(x)\big)\big]^2\diff x \\
		\tilde{\kappa}_2(f_p)&= \int_0^{1}\big(1+pg(x)\big)\big[\log \big(1+pg(x)\big)\big]^2\diff x \\
		{\kappa}_3(f_p)&=\int_0^{1}\big[\log \big(1+pg(x)\big)\big]^3\diff x\\
		\tilde{\kappa}_3(f_p)&=\int_0^{1}\big(1+pg(x)\big)\big[\log \big(1+pg(x)\big)\big]^3\diff x.
	\end{align*}
\end{lemma}
\begin{proof}[Proof of \Cref{thm:chi_square}]
	It suffices to compute the limits in the asymmetric Central Limit \Cref{thm:CLT}, namely
	$$T\cdot \big(\kl(f_p)+\tilde{\kl}(f_p)\big), \,\,T\cdot \kappa_2(f_p),\,\, T\cdot \tilde{\kappa}_2(f_p),\,\, T\cdot {\kappa}_3(f_p)\,\text{ and }\, T\cdot \tilde{\kappa}_3(f_p).$$
	Since $T\sim p^{-2}$, we can consider $p^{-2}\big(\kl(f_p)+\tilde{\kl}(f_p)\big)$ and so on.

	As in \Cref{lem:g}, let $g(x) = -f'(x)-1 = |f'(x)|-1$. The assumption expressed in terms of $g$ is simply
	$$\int_0^1 g(x)^4\diff x<+\infty.$$
	In particular, it implies $|g(x)|^k$ is integrable in $[0,1]$ for $k=2,3,4$. In addition, by \Cref{lem:g}, 
	$$\chi^2(f) = \int_0^{1}\big(f'(x)\big)^2\diff x-1= \int_0^{1}\big(f'(x)+1\big)^2\diff x= \int_0^{1}g(x)^2\diff x.$$

	For the functional $\kl$, by \Cref{lem:g},
	\begin{align}
		\lim_{p\to0^+}\frac{1}{p^2}\,\big(\kl(f_p)+\tilde{\kl}(f_p)\big)
		&= \lim_{p\to0^+}\int_0^{1}g(x)\cdot \frac{1}{p}\log \big(1+pg(x)\big)\diff x\label{eq:sgd1}\\
		&= \int_0^{1}g(x)\cdot \lim_{p\to0^+}\frac{1}{p}\log \big(1+pg(x)\big)\diff x\nonumber\\
		&=\int_0^{1}g(x)^2\diff x=\chi^2(f)\nonumber
	\end{align}
	Changing the order of the limit and the integral in \eqref{eq:sgd1} is approved by the \DCT. To see this, notice that $\log(1+x)\leqslant x$.
	The integrand in \eqref{eq:sgd1} satisfies
	\begin{align*}
		0\leqslant g(x)\cdot \frac{1}{p}\log \big(1+pg(x)\big)\leqslant g(x)^2.
	\end{align*}
	We already argued that $g(x)^2$ is integrable, so it works as a dominating function and the limit is justified. When $p\sqrt{T}\to \nu$, we have
	\[T\cdot \kl(f_p)\to \nu^2\cdot\chi^2(f).\]
	So the constant $K$ in \Cref{thm:CLT} is $\nu^2\cdot\chi^2(f)$.

	For the functional $\kappa_2$ we have
	\begin{align*}
		\frac{1}{p^2}\kappa_2(f_p) &= \int_0^{1}\Big[\frac{1}{p}\log \big(1+pg(x)\big)\Big]^2\diff x.
	\end{align*}
	By a similar dominating function argument,
	\begin{align*}
		\lim_{p\to0^+}\frac{1}{p^2}\,\kappa_2(f_p) = \lim_{p\to0^+}\frac{1}{p^2}\,\tilde{\kappa}_2(f_p) = \int_0^{1}g(x)^2\diff x=\chi^2(f).
	\end{align*}
	Adding in the limit $p\sqrt{T}\to \nu$, we know $s^2$ in \Cref{thm:CLT} is $\nu^2\cdot\chi^2(f)$.

	The same argument involving $\big|g(x)\big|^3$ and $g(x)^4$ applies to the functional $\kappa_3$ and $\tilde{\kappa}_3$ respectively and yields
	$$\lim_{p\to0^+}\frac{1}{p^3}\,\kappa_3(f_p) =\lim_{p\to0^+}\frac{1}{p^3}\,\tilde{\kappa}_3(f_p) = \int_0^{1}g(x)^3\diff x.$$
	Note the different power in $p$ in the denominator. It means $\kappa_3(f_p) = o(p^2)$ and hence $T\cdot \kappa_3(f_p)\to 0$ when $p\sqrt{T}\to \nu$.

	Hence all the limits in \Cref{thm:CLT} check and we have a $G_\mu$ limit where
	\[\mu = K/s = s = \sqrt{\nu^2\cdot\chi^2(f)} = \nu\cdot\sqrt{\chi^2(f)}.\]
	This completes the proof.
\end{proof}
We finish the section by proving the formula in \Cref{lem:chi_G}.
\begin{proof}[Proof of \Cref{lem:chi_G}]
The best calculation is done via better understanding. We point out that the functional $\chi^2$ is doing nothing more than computing the famous $\chi^2$-divergence.
Recall that Neyman $\chi^2$-divergence (reverse Pearson) of $P,Q$ is defined as
\[\chi^2(P\|Q):= \E_P\big[(\tfrac{\diff Q}{\diff P}-1)^2\big]\]
\begin{lemma} \label{lem:}
	If $f=T(P,Q)$ and $f(0)=1$, $f(x)>0$, for all $x<1$, then $\chi^2(f) = \chi^2(P\|Q)$.
\end{lemma}
This lemma is a straightforward corollary of  Proposition B.4 in \cite{dong2019gaussian}, which gives expressions for all $F$-divergence\footnote{We use capital $F$ to avoid confusion with the notation of trade-off function.}. In particular, if $f=T(P,Q)$ and $f(0)=1$, $f(x)>0, \forall x<1$, then $F$-divergence of $P,Q$ can be computed from their trade-off function as follows:
\[D_F(P\|Q) = \int_0^1 F\big({\big|f'(x)\big|}^{-1}\big)\cdot \big|f'(x)\big| \diff x.\]
Neyman $\chi^2$-divergence corresponds to $F(t) = \frac{1}{t}-1$, so
\begin{align*}
	\chi^2(P\|Q) &= \int_0^1 \bigg(\frac{1}{\big|f'(x)\big|^{-1}}-1\bigg)\cdot \big|f'(x)\big| \diff x\\
	&=\int_0^1 \big(f'(x)\big)^2 \diff x- \int_0^1 \big|f'(x)\big| \diff x\\
	&=\int_0^1 \big(f'(x)\big)^2 \diff x- 1.
\end{align*}
With this formula, computing $\chi^2(G_{1/\sigma})$ is straightforward:
\[
	\chi^2(G_{1/\sigma}) = \chi^2\big(\mathcal{N}(\tfrac{1}{\sigma},1)\|\mathcal{N}(0,1)\big) = \e^{1/\sigma^2}-1.
\]
\end{proof}
\subsection{Proof of \Cref{thm:f_compare,thm:better}} 
\label{sub:proof_of_thm:better}

Recall that \Cref{thm:f_compare,thm:better} compare our \texttt{CLT} approach to moments accountant ($\MA$) from two different perspectives: $f$-DP perspective in \Cref{thm:f_compare} and $(\ep,\delta)$-DP perspective in \Cref{thm:better}. We first show that \Cref{thm:f_compare} can be derived from \Cref{thm:better}. Then we prove a refined version of \Cref{thm:better}.
To be more precise about the statement, let us first expand the notations used in the main text.

Let $\delta_\texttt{MA}(\ep;\sigma,p,T)$ be the $\delta$ value computed by moment accountant method (described in detail below) for \texttt{NoisySGD} algorithm with subsampling probability $p$, iteration $T$ and
noise scale $\sigma$. Similarly, $\delta_\texttt{CLT}(\ep;\sigma,\nu)$ denotes the $\delta$ value computed for the same algorithm using central limit theorem assuming $p\sqrt{T}\to \nu$. 

Let $f_T(\alpha)=\sup_{\epsilon \ge 0} f_{\epsilon, \delta_{\mathtt{MA}}(\epsilon)}(\alpha)$. It is supported by $f_{\epsilon_T, \delta_{\mathtt{MA}}(\epsilon_T)}$ at $\alpha$. \Cref{thm:better} says this supporting function is 
smaller than that of $G_{\mu_{\mathtt{CLT}}}$ at $\alpha$ by a strict gap. Taking the limit, $\limsup_{T \to \infty} f_T(\alpha)$ has at least that much gap from $G_{\mu_{\mathtt{CLT}}}(\alpha)$, which proves \Cref{thm:f_compare}.

\Cref{thm:better} is a straightforward corollary of the following proposition. Note that the inequality is reversed compared to the statement of \Cref{thm:better} so that the gap is positive, which also turns $\limsup$ into $\liminf$.

\begin{proposition} \label{thm:diff_exact}
\[
	\liminf_{T\to\infty} \,\,\delta_{\MA}(\ep;\sigma,\tfrac{\nu}{\sqrt{T}},T) - \delta_{\mathtt{CLT}}(\ep;\sigma,\nu) \geqslant \e^{\ep}\cdot\Phi(-\frac{\ep}{\mu}-\frac{\mu}{2})
\]
where $\mu=\nu\cdot \sqrt{\e^{1/\sigma^2}-1}$.
\end{proposition}

Let us first describe how the two methods compute $\delta$ from $\ep$.
\begin{align*}
	\delta_{\nMA}(\ep;\sigma,p,T) &:= \inf_{\lambda\in \textup{orders}} \exp\left( T\cdot\alpha_{\textup{GM}}(\lambda;\sigma,p) - \lambda \epsilon \right)\\
	\delta_{\MA}(\ep;\sigma,p,T) &:= \inf_{\lambda>0} \exp\left( T\cdot\alpha_{\textup{GM}}(\lambda;\sigma,p) - \lambda \epsilon \right)
\end{align*}
where $\alpha_{\textup{GM}}(\lambda;\sigma,p)$ is a scaled version of the R\'enyi divergence of Gaussian mixtures. More specifically, let $P = \mathcal{N}(0,1)$ and $Q=\mathcal{N}(\frac{1}{\sigma},1)$. We further denote the Gaussian mixture $p Q + (1-p) P$ by $\GM_{p,\sigma}$. The 
\begin{align*}
	\alpha_{\textup{GM}}(\lambda;\sigma,p)
	=\max\Big\{\lambda D_{\lambda+1}(\GM_{p,\sigma}\|P), \lambda D_{\lambda+1}(P\|\GM_{p,\sigma})\Big\}.
\end{align*}
In \cite{deep}, it has been shown that \Cref{alg:dpsgd1} (hence also the Adam variant, \Cref{alg:dpsgd2}) with subsampling probability $p$, iteration $T$ and
noise scale $\sigma$ is $(\ep,\delta)$-DP for each $\ep\geqslant0$ if $\delta=\delta_{\MA}(\ep;\sigma,p,T)$. To evaluate the infimum, the domain is discretized\footnote{Code in tensorflow/privacy discretizes at $[1.25, 1.5, 1.75, 2., 2.25, 2.5, 3., 3.5, 4., 4.5, 5, 6, 7, \ldots, 63, 64, 128, 256, 512]$.}. This results in the numerical moment accountant method that is actually implemented. Since $\delta_{\nMA}(\ep;\sigma,p,T)\geqslant\delta_{\MA}(\ep;\sigma,p,T)$, \Cref{alg:dpsgd1} is also $(\ep,\delta)$-DP with $\delta=\delta_{\nMA}(\ep;\sigma,p,T)$.

On the other hand, $\delta_{\texttt{CLT}}(\ep;\sigma,\nu)$ is obtained by first observing \Cref{alg:dpsgd1} is asymptotically $\mu_{\texttt{CLT}}$-GDP with $\mu_{\texttt{CLT}} = \nu\cdot \sqrt{\e^{1/\sigma^2}-1}$ and then convert GDP to $(\ep,\delta)$-DP via \Cref{eq:dual}, i.e. \Cref{alg:dpsgd1} asymptotically satisfies $(\ep,\delta)$-DP where

$$\delta = \delta_{\texttt{CLT}}(\ep;\sigma,\nu) =1 + G^*_{\mu_{\texttt{CLT}}}(-\e^{\epsilon}).
$$

We have just explained how $\MA$ and $\texttt{CLT}$ works. Next we prove \Cref{thm:diff_exact}
\begin{proof}[Proof of \Cref{thm:diff_exact}]

Let $f_{T}=\big(pG_\mu+(1-p)\Id\big)^{\otimes T}$.
We need a lemma (whose proof is provided later) that relates the R\'enyi divergence to the trade-off function $f_{T}$.
\begin{lemma} \label{lem:renyi}
\begin{align*}
	T\cdot\alpha_{\textup{GM}}(\lambda;\sigma,p)&\geqslant\log \int_0^1 |f_T'(x)|^{\lambda + 1} \diff x.
\end{align*}
\end{lemma}
Let $x_T\in(0,1)$ be the point such that $f_T'(x_T) = - \e^\epsilon$ (or $x_T\in\partial f_T^*(- \e^\epsilon)$ if readers worry about differentiability). We have
\[
1 + f_T^*(-\e^{\epsilon}) = \sup_{0 \le x \le 1} \{1 - f_T(x) - \e^{\epsilon} x  \} = 1 - f_T(x_T) - \e^{\epsilon} x_T
\]
It is clear that $|f_T'(x)| \ge \e^\epsilon$ for $0 \le x \le x_T$.

On the other hand, using \Cref{lem:renyi}, we get
\begin{align*}
	\delta_\MA(\ep;\sigma,p,T) &= \inf_{\lambda>0} \exp\left( T\cdot\alpha_{\textup{GM}}(\lambda;\sigma,p) - \lambda \epsilon \right)\\
	& \geqslant \inf_{\lambda>0} \e^{-\lambda \epsilon}\int_0^1 |f_T'(x)|^{\lambda + 1}  \diff x\\
	& >  \inf_{\lambda>0} \int_0^{x_T} |f_T'(x)|^{\lambda + 1} \e^{-\lambda \epsilon} \diff x\\
	& =  \inf_{\lambda>0} \int_0^{x_T} |f'_T(x)| \cdot |f'_T(x)|^{\lambda} \e^{-\lambda \epsilon} \diff x\\
	& \ge  \inf_{\lambda>0} \int_0^{x_T} |f'_T(x)| \cdot (\e^{\epsilon})^{\lambda} \e^{-\lambda \epsilon} \diff x\\
	&= f_T(0) - f_T(x_T)\\
	&= 1 - f_T(x_T)\\
	&= \big(1 - f_T(x_T) - \e^{\epsilon} x_T\big) + \e^{\epsilon} x_T\\
	&= 1 + f_T^*(-\e^{\epsilon})+\e^{\epsilon} x_T.
\end{align*}
In summary, we have
\begin{align}\label{lazyname}
	\delta_\MA(\ep;\sigma,p,T) > 1 + f_T^*(-\e^{\epsilon})+\e^{\epsilon} x_T.
\end{align}
Setting $p=\frac{\nu}{\sqrt{T}}$, we would like to take limit on both sides of \eqref{lazyname}. First notice that $f_T$ converge pointwise to $G_{\mu_{\texttt{CLT}}}$, which we have already proven in \Cref{sub:unjustified_claim_in_sec:noisy-sgd}. The limit of $x_T$ is taken care of in the following lemma:
\begin{lemma} \label{lem:xT}
	\[
 \lim_{T\to\infty} x_T = x^* := \Phi(-\frac{\ep}{\mu_{\texttt{\textup{CLT}}}}-\frac{\mu_{\texttt{\textup{CLT}}}}{2}).
	\]
\end{lemma}
Combining these results, we can take limits on both sides of \eqref{lazyname}:
\begin{align*}
	\liminf_{T\to\infty} \delta_\MA(\ep;\sigma,p,T) &\geqslant \lim_{T\to\infty} 1 + f_T^*(-\e^{\epsilon})+\e^{\epsilon} x_T\\
	&= 1+G_{\mu_{\texttt{CLT}}}^*(-\e^{\epsilon})+\e^{\epsilon} x^*\\
	&=\delta_{\texttt{CLT}}(\ep;\sigma,\nu)+\e^{\epsilon}\cdot \Phi(-\frac{\ep}{\mu_{\texttt{CLT}}}-\frac{\mu_{\texttt{CLT}}}{2}).
\end{align*}
This finishes the proof.
\end{proof}

\begin{proof}[Proof of \Cref{lem:renyi}]
The R\'enyi divergence can also be computed from the trade-off function, just like the $\chi^2$-divergence. In fact, under the same assumptions as in \Cref{lem:chi_G}, we have
\[
D_\alpha(Q\|P) = \frac{1}{\alpha-1}\log \int_0^1|f'(x)|^{\alpha-1}\diff x.
\]
Alternatively,
\begin{equation}\label{eq:renyi}
\lambda D_{\lambda+1}(Q\|P) =\log \int_0^1|f'(x)|^{\lambda}\diff x.
\end{equation}
This identity will be the bridge between $\alpha_{\textup{GM}}$ and $f_T$.

On one hand,
$\alpha_{\textup{GM}}(\lambda;\sigma,p)$ is the maximum of two R\'enyi divergences, so
\begin{align*}
	\alpha_{\textup{GM}}(\lambda;\sigma,p)
	\geqslant\lambda D_{\lambda+1}(p Q + (1-p) P\|P)
\end{align*}
Consequently,
\begin{align*}
	T\cdot\alpha_{\textup{GM}}(\lambda;\sigma,p)
	&\geqslant T\lambda D_{\lambda+1}(p Q + (1-p) P\|P)\\
	&=\lambda D_{\lambda+1}\big((p Q + (1-p) P)^T\|P^T).
\end{align*}
The last step is the tensorization identity of R\'enyi divergence.

On the other hand, notice that $pG_\mu+(1-p)\Id=T\big(P, \GM_{p,\sigma}\big)$ where we continue the use of notations $P = \mathcal{N}(0,1), Q=\mathcal{N}(\frac{1}{\sigma},1)$ and $\GM_{p,\sigma} = p Q + (1-p) P$. We have
\begin{align*}
	f_T
	=\big(pG_\mu+(1-p)\Id\big)^{\otimes T}
	=T\big(P, (pQ + (1-p)P)\big)^{\otimes T}
	=T\big(P^T, (pQ + (1-p)P)^T\big)
\end{align*}
Using \eqref{eq:renyi}, we have
\begin{align*}
	T\cdot\alpha_{\textup{GM}}(\lambda;\sigma,p)
	&\geqslant\lambda D_{\lambda+1}\big((p Q + (1-p) P)^T\|P^T)\\
	&=\log \int_0^1|f_T'(x)|^{\lambda}\diff x.
\end{align*}
\end{proof}

\begin{proof}[Proof of \Cref{lem:xT}]
	By definition, $f_T'(x_T) = -\e^\ep$. The convexity of $f_T$ implies $\nabla f_T^*(-\e^\ep) = x_T$.

	Since $f_T$ converges uniformly to $G_{\mu_{\texttt{\textup{CLT}}}}$ in $[0,1]$, we have uniform convergence $f_T^*\to G_{\mu_{\texttt{\textup{CLT}}}}^*$. By convexity of these functions, the convergence also implies the convergence of derivatives
	(See Theorem 25.7 of \cite{rockafellar1970convex}), 
	namely,
	\[\nabla f_T^*\to \nabla G_{\mu_{\texttt{\textup{CLT}}}}^*.\]
	Therefore,
	\begin{align*}
		x_T = \nabla f_T^*(-\e^\ep)\to \nabla G_{\mu_{\texttt{\textup{CLT}}}}^*(-\e^\ep).
	\end{align*}
	Let $x^*=\nabla G_{\mu_{\texttt{\textup{CLT}}}}^*(-\e^\ep)$ be the limit. Using the convexity again, we have
	\begin{align*}
		-\e^\ep= G_{\mu_{\texttt{\textup{CLT}}}}'(x^*).
	\end{align*}
	We can solve for $x^*$ using the expression of $G_\mu$ \eqref{eq:dual}. After some algebra, we have
	\[x^* = \Phi(-\frac{\ep}{\mu_{\texttt{\textup{CLT}}}}-\frac{\mu_{\texttt{\textup{CLT}}}}{2}).\]
	The proof is complete.
\end{proof}



\end{document}